\documentclass[11pt, a4paper]{amsart}

\usepackage{amssymb, amsmath, amsthm, enumerate, fancyhdr} 
\usepackage[utf8]{inputenc}
\usepackage[T1]{fontenc}
\usepackage{ifthen}
\usepackage{tikz}
\usetikzlibrary{shapes,fit,automata,decorations.pathmorphing}

\usepackage{cancel, cases, mathabx, fullpage}
\usepackage{times}

\newcommand{\makePageBig}{
\setlength{\topmargin}{-1cm}
\setlength{\oddsidemargin}{-7mm}
\setlength{\evensidemargin}{-7mm}
\textwidth18cm
\textheight235mm
\headheight14pt
\headsep5mm}

\providecommand{\ignore}[1]{}

\newcommand{\ORT}{\textbf{ORT}\xspace}

\newcommand{\Txt}[1][]{\mathbf{Txt}^{#1}}
\newcommand{\Inf}{\mathbf{Inf}}
\newcommand{\can}{\mathrm{can}}

\newcommand{\It}{\mathbf{It}}

\newcommand{\Fin}{\mathbf{Fin}}
\newcommand{\Ex}{\mathbf{Ex}}
\newcommand{\Lim}{\mathbf{Ex}}

\newcommand{\Cons}{\mathbf{Cons}}
\newcommand{\Comp}{\mathbf{Cons}}

\newcommand{\Conv}{\mathbf{Conv}}

\newcommand{\LocConv}{\mathbf{LocConv}}
\newcommand{\Caut}{\mathbf{Caut}}

\newcommand{\NU}{\mathbf{NU}}
\newcommand{\SNU}{\mathbf{SNU}}
\newcommand{\Dec}{\mathbf{Dec}}
\newcommand{\SDec}{\mathbf{SDec}}
\newcommand{\Wb}{\mathbf{Wb}}

\newcommand{\WMon}{\mathbf{WMon}}
\newcommand{\Mon}{\mathbf{Mon}}
\newcommand{\SMon}{\mathbf{SMon}}

\newcommand{\N}{\mathbb{N}}
\newcommand{\Z}{\mathbb{Z}}

\newcommand{\Seq}[1]{\text{$#1^{<\omega}$}}
\newcommand{\ISeq}[1]{\text{$#1^{\omega}$}}
\newcommand{\IfiSeq}[1]{\text{$#1^{\leq \omega}$}}
\newcommand{\concat}{{^\smallfrown}}
\newcommand{\ps}{\mathrm{pos}}
\renewcommand{\ng}{\mathrm{neg}}

\newcommand{\partialFn}{\mathfrak{P}}

\newcommand{\totalCp}{\mathcal{R}}

\newcommand*{\Nttwo}{\N\!\times\!\{0,1\}}
\newcommand*{\inseg}{\sqsubseteq}

\newcommand{\pad}{\mathrm{pad}}

\newcommand{\last}{\mathrm{last}}
\newcommand{\dom}{\mathrm{dom}}
\newcommand{\ran}{\mathrm{ran}}
\newcommand{\cnt}{\mathrm{content}}

\newcommand{\ind}{\mathrm{ind}}
\newcommand{\simu}{\mathfrak{s}}
\newcommand{\nonsimu}{\mathfrak{r}}
\newcommand{\Simu}{\mathfrak{S}}

\newcommand{\pr}{\mathrm{pr}}

\newcommand{\sqr}[2]{{\vcenter{\hrule height.#2pt
        \hbox{\vrule width.#2pt height#1pt \kern#1pt
           \vrule width.#2pt}
        \hrule height.#2pt}}}

\newcommand{\Qed}[1]{{\nobreak\hfil\penalty50
  \hskip1em\hbox{}\nobreak\hfil$\sqr{8}{8}$\enspace\sc #1
  \parfillskip=0pt \finalhyphendemerits=0 \par}
  \bigskip}

\newcommand{\QedFor}[1]{\Qed{(for~#1)}}
\newcommand{\QedForClaim}[1][]{\QedFor{~Claim#1}}

\renewenvironment{proof}{

\noindent
\noindent\emph{Proof.}\enspace}{\Qed{}}

\newcommand{\CalC}{\mathcal{C}}

\newcommand{\CalH}{\mathcal{H}}
\newcommand{\CalI}{\mathcal{I}}

\newcommand{\CalL}{\mathcal{L}}

\newcommand{\CalR}{\mathcal{R}}

\newcommand*{\LOCK}{\mathrm{LOCK}}

\renewcommand*{\u}{\mathbf{u}}
\newcommand*{\ve}{\mathbf{v}}
\newcommand*{\w}{\mathbf{w}}
\newcommand*{\x}{\mathbf{x}}
\newcommand*{\y}{\mathbf{y}}
\newcommand*{\z}{\mathbf{z}}

\newcommand{\yes}{{\textbf{1}}\xspace\;\;}
\newcommand{\no}{\textcolor{gray}{{\textbf{0}}}\xspace\;\;}

\newtheorem{theorem}{Theorem}[section]

\newtheorem{corollary}[theorem]{Corollary}
\newtheorem{lemma}[theorem]{Lemma}

\newtheorem{lem}[theorem]{Lemma}

\theoremstyle{definition}
\newtheorem{definition}[theorem]{Definition}
\newtheorem{defn}[theorem]{Definition}

\newcounter{claimCounter}[theorem]

\newcounter{caseCounter}[theorem]

\newcounter{subcaseCounter}[caseCounter]

\newboolean{useproof}
\setboolean{useproof}{false}

\def\makeinnocent#1{\catcode`#1=12 }
\def\csarg#1#2{\expandafter#1\csname#2\endcsname}

\def\ThrowAwayComment#1{\begingroup
    \def\CurrentComment{#1}    \let\do\makeinnocent \dospecials
    \makeinnocent\^^L    \endlinechar`\^^M \catcode`\^^M=12 \xComment}
{\catcode`\^^M=12 \endlinechar=-1  \gdef\xComment#1^^M{\def\test{#1}
      \csarg\ifx{PlainEnd\CurrentComment Test}\test
          \let\next\endgroup
      \else \csarg\ifx{LaLaEnd\CurrentComment Test}\test
            \edef\next{\endgroup\noexpand\end{\CurrentComment}}
      \else \let\next\xComment
      \fi \fi \next}
}

{\escapechar=-1\relax
	\csarg\xdef{PlainEndproofTest}{\string\\endproof}	\csarg\xdef{LaLaEndproofTest}{\string\\end\string\{proof\string\}}}

\renewcommand{\proof}{\ifthenelse{\boolean{useproof}}{

\noindent
\noindent\emph{Proof.}}{\ThrowAwayComment{proof}}}

\renewcommand{\endproof}{\ifthenelse{\boolean{useproof}}{\Qed{}}{}}

\makePageBig

\title[An Iterative Learner for Half-Spaces]{Learning Half-Spaces and other Concept Classes \\ in the Limit with Iterative Learners}

\author{%
 Ardalan Khazraei, Timo Kötzing, Karen Seidel%
}

\tikzset{>=latex}

\setboolean{useproof}{true}

\newcommand{\state}{\mathit{State}}
\newcommand{\open}{\mathit{Open}}
\newcommand{\locked}{\mathit{Locked}}

\DeclareMathOperator{\lcm}{lcm}
\DeclareMathOperator{\conv}{conv}
\DeclareMathOperator{\pos}{pos}
\DeclareMathOperator{\negative}{neg}

\usepackage[algo2e]{algorithm2e}
\SetKwFor{Function}{function}{is}{end function}
\newcommand{\assign}{\leftarrow}
\SetKw{Input}{input}
\SetKw{Output}{output}
\SetKw{Initialization}{initialization}

\begin{document}

\begin{abstract}
In order to model an efficient learning paradigm, iterative learning algorithms access data one by one, updating the current hypothesis without regress to past data.
Past research on iterative learning analyzed for example many important additional requirements and their impact on iterative learners.

In this paper, our results are twofold.
First, we analyze the relative learning power of various settings of iterative learning, including learning from text and from informant, as well as various further restrictions, for example we show that strongly non-U-shaped learning is restrictive for iterative learning from informant.

Second, we investigate the learnability of the concept class of half-spaces and provide a constructive iterative algorithm to learn the set of half-spaces from informant.
\end{abstract}

\maketitle

\setboolean{useproof}{true}

\newcommand{\justify}[1]{\textcolor{red}{#1}}

\newcommand{\drawBackboneInf}{
\begin{scope}[every node/.style={minimum size=4mm}]

\node (nothing) at (0,0) {\textbf{T}};
\node (nu)  at (3,-1)     	{$\NU$};
\node (dec)  at (1,-3)   		{$\Dec$};

\node (smon) at (-5,-8)		  {$\SMon$};
\node (mon) at (-5,-6)		  {$\Mon$}; 
\node (wmon) at (-1.5,-3.5)     {$\WMon$};
\node (caut) at (-2,-2)    	{$\Caut$};

\node (sdec) at (3,-5)   		{$\SDec$};
\node (snu) at (5,-3)     	{$\SNU$};
\node (conv) at (5,-7)     	{$\Conv$};

\draw (smon) -- (dec);
\draw (snu) -- (nu);
\draw (nu) -- (dec);
\draw (dec) -- (sdec); 
\draw (sdec) -- (snu);
\draw (snu) -- (conv);
\draw (nothing) -- (caut);
\draw (caut) -- (smon);
\draw (nothing) -- (nu);
\draw (nothing) -- (wmon);
\draw (nothing) edge[bend right] (mon);
\draw (mon) -- (smon);
\draw (wmon) edge[bend right=25] (conv);
\draw (wmon) -- (smon);
\end{scope}
}

\newcommand{\drawBackboneTwo}{
\begin{scope}[every node/.style={minimum size=4mm}]

\node (nothing) at (0,0) {\textbf{T}};
\node (nu)  at (3,-1)     	{$\NU$};
\node (dec)  at (1,-3)   		{$\Dec$};

\node (smon) at (-5,-8)		  {$\SMon$};
\node (mon) at (-5,-6)		  {$\Mon$}; 
\node (wmon) at (-1.5,-3.5)     {$\WMon$};
\node (caut) at (-2,-2)    	{$\Caut$};

\node (sdec) at (3,-5)   		{$\SDec$};
\node (snu) at (5,-3)     	{$\SNU$};
\node (conv) at (5,-7)     	{$\Conv$};

\draw[<-] (dec) -- (smon);
\draw[<-] (nu) -- (snu);
\draw[<-] (nu) -- (dec);
\draw[<-] (dec) -- (sdec); 
\draw[<-] (snu) -- (sdec);
\draw[<-] (snu) -- (conv);
\draw[<-] (nothing) -- (caut);
\draw[<-] (caut) -- (smon);
\draw[<-] (nothing) -- (nu);
\draw[<-] (nothing) -- (wmon);
\draw[<-] (nothing) edge[bend right] (mon);
\draw[<-] (mon) -- (smon);
\draw[<-] (wmon) edge[bend right=25] (conv);
\draw[<-] (wmon) -- (smon);
\end{scope}
}

\bigskip

\section{Introduction}

We are interested in the problem of algorithmically learning a description for a formal language (a computably enumerable subset of the set of natural numbers) when presented successively all and only the elements of that language; this is sometimes called \emph{inductive inference}, a branch of (algorithmic) learning theory.
For example, a learner $M$ might be presented more and more even numbers. After each new number, $M$ outputs a description for a language as its conjecture. The learner $M$ might decide to output a program for the set of all multiples of $4$, as long as all numbers presented are divisible by~$4$. Later, when $h$ sees an even number not divisible by $4$, it might change this guess to a program for the set of all multiples of~$2$.

Many criteria for deciding whether a learner $M$ is \emph{successful} on a language~$L$ have been proposed in the literature. Gold, in his seminal paper \cite{Gol:j:67}, gave a first, simple learning criterion, \emph{$\Txt\Ex$-learning}\footnote{$\Txt$ stands for learning from a \emph{text} of positive examples; $\Ex$ stands for \emph{explanatory}.}, where a learner is \emph{successful} iff, on every \emph{text} for $L$ (a listing of all and only the elements of $L$) it eventually stops changing its conjectures, and its final conjecture is a correct description for the input sequence.
Trivially, each single, describable language $L$ has a suitable constant function as a $\Txt\Ex$-learner (this learner constantly outputs a description for $L$).
As we want algorithms for more than a single learning task, we are interested in analyzing for which \emph{classes of languages} $\CalL$ is there a \emph{single learner} $M$ learning \emph{each} member of $\CalL$. This framework is also sometimes known as \emph{language learning in the limit} and has been studied extensively, using a wide range of learning criteria similar to $\Txt\Ex$-learning (see, for example, the textbook \cite{Jai-Osh-Roy-Sha:b:99:stl2}).

One major criticism of the model suggested by Gold is its excessive use of memory: for each new hypothesis the entire history of past data is available. Iterative learning is the most common variant of learning in the limit which addresses memory constraints: the memory of the learner on past data is just its current hypothesis. Due to the padding lemma, this memory is still not void, but finitely many data can be memorized in the hypothesis. 

There is already a quite comprehensive body of work on iterative learning \cite{Cas-Koe:c:10:colt,Cas-Moe:j:08:nonuit,%
jain2016role,Jai-Moe-Zil:j:13,Jai-Osh-Roy-Sha:b:99:stl2}. However, this work focuses on learning from from text, that is, from positive data only.
In this paper we are also interested in the other important paradigm of learning from both positive and negative information.
For example, when learning half-spaces, one could see data declaring that $\langle 1,1\rangle$ is in the target half-space, further is $\langle 3,2\rangle$, but $\langle 1,10\rangle$ is not, and so on. This setting is called \emph{learning from informant} (in contrast to learning from \emph{text}).

\begin{center}
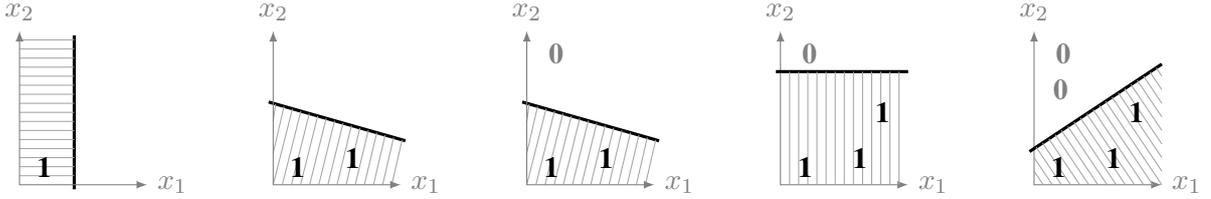
\begin{figure}[h]%
\qquad
\begin{minipage}{.18\textwidth}
\begin{tikzpicture}[scale=0.12]
	\draw [ very thick] (6,-0.5) -- (6,16.5);  
	\foreach \y in {1,...,16} {
	\draw [gray!70, thin] (6,\y) -- (0,\y); 
	}
	\draw [gray,<->,thin] (0,17) node (yaxis) [above] {\({x_{2}}\)} |- (14,0) node (xaxis) [right] { \({x_{1}}\)};
	\node[inner sep=0pt] (a) at (4,2)	{\yes};
\end{tikzpicture}
\end{minipage}
\begin{minipage}{.18\textwidth}
\begin{tikzpicture}[scale=0.12]
			\pgfmathsetmacro{\m}{-4/14}
			\pgfmathsetmacro{\b}{9}
			\draw [ very thick] (-0.5,-0.5*\m+\b) -- (14.5,14.5*\m+\b);  
			\foreach \u in {0,...,13} {
				\pgfmathsetmacro{\sx}{196/212*\u+112/53}
				\draw [gray!70, thin] (\u,0) -- (\sx,\sx*\m+\b); 
			}
			\foreach \u in {-2,-1} {
				\pgfmathsetmacro{\sx}{196/212*\u+112/53}
				\draw [gray!70, thin] (0,-14/4*\u) -- (\sx,\sx*\m+\b); 
			}
			\draw [gray,<->,thin] (0,17) node (yaxis) [above] { \({x_{2}}\)} |- (14,0) node (xaxis) [right] { \({x_{1}}\)};
			\node[inner sep=0pt] (a) at (4,2)
				{\yes};
			\node[inner sep=0pt] (b) at (10,3)
				{\yes};
		\end{tikzpicture}
\end{minipage}
\begin{minipage}{.18\textwidth}
		\begin{tikzpicture}[scale=0.12]
			\pgfmathsetmacro{\m}{-4/14}
			\pgfmathsetmacro{\b}{9}
			\draw [ very thick] (-0.5,-0.5*\m+\b) -- (14.5,14.5*\m+\b);  
			\foreach \u in {0,...,13} {
				\pgfmathsetmacro{\sx}{196/212*\u+112/53}
				\draw [gray!70, thin] (\u,0) -- (\sx,\sx*\m+\b); 
			}
			\foreach \u in {-2,-1} {
				\pgfmathsetmacro{\sx}{196/212*\u+112/53}
				\draw [gray!70, thin] (0,-14/4*\u) -- (\sx,\sx*\m+\b); 
			}
			\draw [gray,<->,thin] (0,17) node (yaxis) [above] { \({x_{2}}\)} |- (14,0) node (xaxis) [right] { \({x_{1}}\)};
			\node[inner sep=0pt] (a) at (4,2)
				{\yes};
			\node[inner sep=0pt] (b) at (10,3)
				{\yes};
			\node[inner sep=0pt] (c) at (4.5,14.5) {\no};
		\end{tikzpicture}
\end{minipage}
\begin{minipage}{.18\textwidth}
		\begin{tikzpicture}[scale=0.12]
			\draw [gray,<->,thin] (0,17) node (yaxis) [above] { \({x_{2}}\)} |- (14,0) node (xaxis) [right] { \({x_{1}}\)};
			\draw [ very thick] (-0.5,12.5) -- (14,12.5);  
			\foreach \x in {1,...,13} {
				\draw [gray!70, thin] (\x,12.5) -- (\x,0); 
			}
			\node[inner sep=0pt] (a) at (4,2)
				{\yes};
			\node[inner sep=0pt] (b) at (10,3)
				{\yes};
			\node[inner sep=0pt] (c) at (4.5,14.5)
				{\no};
			\node[inner sep=0pt] (d) at (12.5,8)
				{\yes};
		\end{tikzpicture}
\end{minipage}
\begin{minipage}{.18\textwidth}
		\begin{tikzpicture}[scale=0.12]
			\draw [gray,<->,thin] (0,17) node (yaxis) [above] { \({x_{2}}\)} |- (14,0) node (xaxis) [right] { \({x_{1}}\)};
			\draw [ very thick] (-0.5,3.6667) -- (14,13.3333);  
			\foreach \x in {1,...,2} {
				\draw [gray!70, thin] (\x,0) -- (0,3*\x/2); 
			}
			\foreach \x in {3,...,14} {
				\pgfmathsetmacro{\mx}{9*\x/13 - 24/13}
				\pgfmathsetmacro{\y}{2*\mx /3+4)}
				\draw [gray!70, thin] (\x,0) -- (\mx,\y); 
			}
			\foreach \x in {15,...,22} {
				\pgfmathsetmacro{\mx}{9*\x/13 - 24/13}
				\pgfmathsetmacro{\y}{2*\mx /3+4)}
				\draw [gray!70, thin] (14,-21+3*\x/2) -- (\mx,\y); 
			}
			\node[inner sep=0pt] (a) at (4,2)
				{\yes};
			\node[inner sep=0pt] (b) at (10,3)
				{\yes};
			\node[inner sep=0pt] (c) at (4.5,14.5)
				{\no};
			\node[inner sep=0pt] (d) at (12.5,8)
				{\yes};
			\node[inner sep=0pt] (e) at (4.2,10.5)
				{\no};
		\end{tikzpicture}
\end{minipage}
\caption{Learning Process when the hypotheses correspond to half-spaces and data is binary labeled.}%
\label{LearningProcess}%
\end{figure}
\end{center}

Iterative learning from informant was analyzed by \cite{Jai-Lan-Zil:j:07}, where various natural restrictions were considered; the authors focused on the case of learning indexable families (classes of languages which are uniformly decidable). Here they showed for example that learners can be assumed to be consistent with the data just seen, but not necessarily with all previously presented data, both for learning from text and from informant.
In this paper we additionally consider learning of arbitrary classes of computably enumerable languages and of classes with only recursive languages.

In Section~\ref{sec:TextAndIt} we consider two restrictions on learning from informant: learning from text and learning iteratively. We show that both these restrictions render fewer classes of languages learnable; in fact, the two restrictions yield two incomparable sets of language classes being learnable, which also shows that learning iteratively from text is weaker than supposing just one of the two restrictions.

For understanding iterative learners we analyze what normal forms can be assumed about such learners in Section~\ref{TotCan}. First we show that, analogously to the case of learning from text (as analyzed in~\cite{CaseMoelius2009}), we cannot assume learners to be total (i.e.\ always giving an output). However, from \cite{Cas-Moe:c:07:nonuit} we know that we can assume iterative text learners to be \emph{canny}; we adapt this normal form for the case of iterative learning from informant and show that it can be assumed to hold for iterative learners generally. 

Many works focus on understanding these properties via relating different learning restrictions for the learning setting at hand; for example, \cite{jain2016role} mapped out all pairwise relations for a group of learning restrictions for iterative learning from text. A similar map for the case of iterative learning from informant is not known, but we believe that the normal form of canniness is an important stepping stone to understand iterative learners better and determine such pairwise relations. In Section~\ref{AddRequ} we collect all previously known results for such a map, give more such relations and discuss which questions remain open.

We complement these structural insights with an analysis of the learnability of the language class of half-spaces in Sections~\ref{sec:halfSpacestwo} and \ref{sec:halfSpacesgen}.
Fundamental machine learning algorithms for supervised binary classification like support vector machines and the perceptron use half-spaces as hypothesis space.
With a fixed computable kernel function even more learning tasks can be reduced to classifying with half-spaces.
The learnability of linear predictors has been investigated with respect to other learning models and respective research questions, e.g. PAC-learning \cite{shamir2015sample}, Preference-based Teaching \cite{gao2017preference}. See \cite{shalev2014understanding} for an introduction to this concept class and different implemented learning algorithms.
As we are concerned with computable learners, we first formalize the problem by encoding it appropriately.
We then observe that the set of half-spaces forms an indexable family and is therefore learnable by enumeration from informant by a full-information learner, due to \cite{Gol:j:67}.
Our contribution is a geometric and therefore constructive iterative learning algorithm for the family of half-spaces.
The iterative learner patiently waits for data indicating that he already encountered a locking sequence.
Every so-called $\LOCK$-state directly corresponds to a half-space. In a $\LOCK$ state the learner ignores all further consistent data.
Hence, our iterative learning algorithm employs the option to store data as part of the hypothesis in order to wait for helpful data and on the other hand is smart enough to know, when to stop collecting.
In Section~\ref{sec:halfSpacestwo} we illustrate the algorithm in dimension 2.
The general constructive algorithm and a complete correctness proof for arbitrary dimension can be found in Section~\ref{sec:halfSpacesgen}.


We continue this paper with some mathematical preliminaries in Section~\ref{sec:prelim} before discussing our results in more detail.


\section{Iterative Learning from Informant}\label{sec:prelim}

We let $\N$ denote the \emph{natural numbers} including $0$ and write $\infty$ for an \emph{infinite cardinality}.
Moreover, for a function $f$ we write $\dom(f)$ for its \emph{domain} and $\ran(f)$ for its \emph{range}.
If we deal with (a subset of) a cartesian product, we are going to refer to the \emph{projection functions} to the first or second coordinate by $\pr_1$ and $\pr_2$, respectively.
Further, $\Seq{X}$ denotes the \emph{finite sequence}s over $X$ and $\ISeq{X}$ stands for the \emph{countably infinite sequence}s over $X$.
Additionally, $\IfiSeq{X} := \Seq{X} \cup \ISeq{X}$ denotes the set of all \emph{countably finite or infinite sequence}s over $X$. 
For every $f \in \IfiSeq{X}$ and $t \in \N$, we let
$f[t] := \{ (s,f(s)) \mid s < t \}$ denote the \emph{restriction of $f$ to $t$}.
Finally, for sequences $\sigma, \tau \in \Seq{X}$ their concatenation is denoted by $\sigma\concat\tau$ and we write $\sigma \inseg \tau$, if $\sigma$ is an initial segment of $\tau$, i.e., there is some $t \in \N$ such that $\sigma = \tau[t]$.
Moreover, we concatenate sequences by writing them consecutively.
In our setting, we typically have $X = \N$ or
$X=\N \times \{0,1\}$.

As far as possible, notation and terminology on the learning theoretic side follow \cite{STL1} and \cite{Jai-Osh-Roy-Sha:b:99:stl2}, whereas on the computability theoretic side we refer to \cite{Odi:b:99}, \cite{Rog:b:67} and \cite{Koe:th:09}.

A \emph{language $L$} is a recursively enumerable subset of $\N$.
A \emph{prediction model $f$} is a function $f:\N \to \{0,1\}.$
We identify subsets of $\N$ with their characteristic functions $\N\to\{0,1\}$.
Hence, there is a one-one correspondence between recursive languages and recursive binary functions.
We denote the characteristic function for $L\subseteq\N$ by $f_L$.


When considering binary supervised learning, the \emph{set of all training data sequences $\mathbb{S}$} is the set of all finite sequences
$$\sigma=((n_0,y_0),\ldots,(n_{|\sigma|-1},y_{|\sigma|-1}))$$
of \emph{consistently} binary labeled natural numbers.
In case of learning from positive data only, we encounter the set $\mathbb{T}$ of finite sequences $\tau=(n_0,\ldots,n_{|\tau|-1})$ of natural numbers.

In the context of language learning, \cite{Gol:j:67}, in his seminal paper, distinguished two major different kinds of information presentation.
A function $$I:\N \to \N\times\{0,1\}$$ is an \emph{informant for language $L$}, if there is a surjection $n:\N\to\N$ such that  for every $t\in\N$ holds $I(t)=(n(t),f_L(n(t))).$
As $f_L$ is used to label the range of $n$, only consistently labeled sequences result.
Hence, the range of $I$ is a complete information about $L$ but $I$ is free to repeat data.
Moreover, for an informant $I$ we let
\begin{align*}
\ps(I) &:= \{ y \in \N \mid \exists x \in \N \colon \pr_1(I(x))=y \wedge \pr_2(I(x))=1 \} \text{ and } \\
\ng(I) &:= \{ y \in \N \mid \exists x \in \N \colon \pr_1(I(x))=y \wedge \pr_2(I(x))=0 \}
\end{align*}
denote the sets of all natural numbers, about which $I$ gives some positive or negative information, respectively.

A \emph{text for language $L$} is a function $T:\N \to \N\cup\{\#\}$ with range $L$ after removing $\#$. The symbol $\#$ is interpreted as pause symbol and added to deal with finite languages.
The main difference between an informant and a text for $L$ is that the informant tells you also that a natural number is \emph{not} in $L$.

A set $\CalL=\{L_i \mid i\in\N\}$ of languages is called \emph{indexable family} if there is a computer program that on input \((i, n)\in\N^2\)	returns $1$ if \(n\in L_i\) and $0$ otherwise.
Important examples are $\Fin$ and $\mathbf{CoFin}$, the set of all finite subsets of $\N$ and the set of all complements of finite subsets of $\N$, respectively.
				
A \emph{learner $M$ from informants (texts)} is a (partial) computable function $$M: \mathbb{S} \to \N \qquad (M: \mathbb{T} \to \N)$$ with the output interpreted with respect to a prefixed hypothesis space $\CalH$.

Often the hypothesis space is an indexable class or the established $W$-hypothesis space defined in Subsection~\ref{TotCan}.

Let $\CalL$ be a collection of languages that we want to learn. We will refer to $\CalL$ as the concept class which will often be an indexable family.
Further, let $\CalH=\{ L_i \mid i\in\N\}$ with $\CalL\subseteq\CalH$ be a second collection of languages called the hypothesis space.
In general we do \emph{not} assume that for every $L\in\CalL$ there is a unique index $i\in\N$ with $L_i=L$.
Indeed, ambiguity in the hypothesis space helps memory-resticted learners to remember data.

Let $I$ be an informant ($T$ be a text) for $L$ and $\CalH=\{ L_i \mid i\in\N\}$ a hypothesis space.
A learner $M: \mathbb{S}\to\N$ ($M: \mathbb{T}\to\N$) is \emph{successful on $I$ (on $T$)} if it eventually settles on $i\in\N$ with $L_i=L$.
This means that when receiving increasingly long finite initial segments of $I$ (of $T$) as inputs, it will from some time on be correct and not change the output on longer initial segments of $I$ (of $T$).

$M$ \emph{learns} $L$ if it is successful on every informant $I$ (on every text $T$) for $L$.
$M$ \emph{learns $\CalL$} if there is a hypothesis space $\CalH$ such that $M$ learns every $L\in\CalL$.
We denote the collection of all $\CalL$ learnable from informant (text) by $[\Inf\Ex]$ ($[\Txt\Ex]$).
If we fix the hypothesis space, we denote this by a subscript for $\Ex$.

According to \cite{wiehagen1976limes}, \cite{Lan-Zeu:j:96}, \cite{Cas-Jai-Lan-Zeu:j:99:feedback} a learner $M$ is \emph{iterative} if its output on $\sigma\in \mathbb{S}$ ($\tau\in \mathbb{T}$) only depends on the last input $\last(\sigma)$ and the hypothesis $M(\sigma^-)$ after observing $\sigma$ without its last element $\last(\sigma)$. In this sense the learner forgets all prior data and can only refer to the hypothesis which resulted from this data.
The collection of all $\CalL$ learnable by an iterative learner from informant (text) is denoted by $[\It\Inf\Ex]$ ($[\It\Txt\Ex]$).

\section{Comparison with Learning from Text}\label{sec:TextAndIt}

As every informant incorporates a text for the language presented, we gain $[\It\Txt\Ex]\subseteq[\It\Inf\Ex]$ by ignoring negative information.

It has been observed in \cite{STL1} that the superfinite language class $\Fin\cup\{\N\}$ is in $[\Inf\Ex]\setminus[\It\Inf\Ex]$.
Moreover, with $L_k= 2\N\cup\{2k+1\}$ and $L'_k = L_k\setminus \{2k\}$ the indexable family $\CalL =\{2\N\}\cup \{L_k,L'_k \mid k\in\N \}$ lies in $[\Txt\Ex]\cap[\It\Inf\Ex]$ but not in $[\It\Txt\Ex]$.
In \cite{Jai-Osh-Roy-Sha:b:99:stl2} the separations are witnessed by the indexable family 
$\{\mathbb{N} \setminus \{0\}\}\cup\{ D \cup \{0\} : D \in \mathbf{Fin}\}$.

We already observed that not every indexable family is learnable by an iterative learner from informant.
On the other hand, learning by enumeration makes every indexable family learnable by an iterative learner from the informants labeling all natural numbers in the canonical order, see \cite{Gol:j:67}.

\bigskip
It can easily be verified that $\mathbf{CoFin} \in [\It\Inf\Ex]\setminus[\Txt\Ex]$ and with the next result $[\It\Inf\Ex]\perp[\Txt\Ex]$, where $\perp$ stands for incomparability with respect to set inclusion, meaning 
(1) there is a concept class learnable from text but not by an iterative learner from informant and
(2) there is a concept class learnable by an iterative learner from informant but not from text.


\begin{lemma}
There is an indexable family in $\mathbf{[TxtEx]}\setminus\mathbf{[ItInfEx]}$.
\end{lemma}
\begin{proof}
As there is a computable bijection between $\N$ and $\N\times\N$, we can also consider subsets of $\N\times\N$ as languages.
Denote by $L_{S,D} = S \times (D \cup \{0\}) \cup (\mathbb{N} \setminus S) \times (\mathbb{N} \setminus \{0\}) \subseteq \N\times\N$ the language with $D\cup\{0\}$ in all rows numbered by an $s\in S$ and $\mathbb{N} \setminus \{0\}$ in all other rows.
Consider the indexable family \[ \CalL=\{ L_{S,D} \mid S, D \in \mathbf{Fin}\}.\]

$\mathcal{L}$ is clearly an indexable family, as there is a computable enumaration of all pairs $(S, D)$ where $S$ is a finite subset of $\mathbb{N}$ and $D$ is a finite subset of $\mathbb{N} \setminus \{0\}$. Moreover, there is a uniform procedure to check whether $(n_1,n_2)$ is in $L_{S,D}$.

\textbf{$\mathcal{F} \in \mathbf{[TxtEx]}$:}
Maintain full information at step $n$ of the entire sequence $T[n]$ read from text. Conjecture $S' := \{x | (x, 0) \in T[n]\}$ and $D' := \{y | \exists x \in S' : (x, y) \in T[n]\}$. $S'$ will eventually converge to $S$ as all $(x, 0)$ will be received by the learner at some point for all $x \in S$. After $S' = S$, we can say that $D'$ will also converge to $D$ (if it has not already) because at some point all $(x, y)$ will have been received for all $x \in S$.

\textbf{$\mathcal{F} \not\in \mathbf{[ItInfEx]}$:} Suppose an iterative learner $M$ learns $\mathcal{F}$ from informants. Let $\sigma$ be a locking sequence of $M$ for $\mathbb{N} \times (\mathbb{N} \setminus \{0\})$.
Let $x_0$ be such that $(x_0, 0)$ does not appear in $\sigma$.
Such an $x_0$ must exist because there are infinitely many $(x, 0)$ but $\sigma$ is a finite sequence.
Define $D := \{y| (x_0,y) \in \pos(\sigma)\}$.
$L := \{x_0\} \times (D \cup \{0\}) \cup (\mathbb{N} \setminus \{x_0\}) \times (\mathbb{N} \setminus \{0\})$ is then consistent with $\sigma$, so let $\sigma' \sqsupseteq \sigma$ be a locking sequence for $L$.
Define $y_0$ such that $y_0 > \max\left(\{0\} \cup \{y| \exists x: (x, y) \in \pos(\sigma') \cup \negative(\sigma')\}\right)$.
The element $(x_0, y_0)$ is consistent with $\mathbb{N} \times (\mathbb{N} \setminus \{0\})$ if and only if it is labeled positively and with $L$ if and only if it is labeled negatively.
Because $\sigma$ is a locking sequence for $\mathbb{N} \times (\mathbb{N} \setminus \{0\})$ and $((x_0, y_0), 1))$ is consistent with it, $M(\sigma ((x_0, y_0), 1)) = M(\sigma) = e_1$ such that $W_{e_1} =  \mathbb{N} \times (\mathbb{N} \setminus \{0\})$ so by iterativeness of $M$ we have that if $\tau := \sigma ((x_0, y_0), 1) (\sigma' - \sigma)$ where $\sigma' - \sigma$ is the subsequence of $\sigma'$ starting after $\sigma$ ends, then $M(\tau) = M(\sigma')$ meaning $\tau$ is also a locking sequence for $L$.
This is a contradiction because if $I$ is an informant for $L$, then $J := I \setminus \{((x_0, y_0), 0)\}$ is also consistent with $L$ so for all $\ell \ge 0$ we have $M(\tau J[\ell]) = M(\sigma') = e_2$ such that $W_{e_2} = L$ but $\tau J$ is an informant for $L':= \{x_0\} \times (D \cup \{(x_0, y_0)\} \cup \{0\}) \cup (\mathbb{N} \setminus \{x_0\}) \times (\mathbb{N} \setminus \{0\}) \in \mathcal{F}$ and $L' \ne L$, a contradiction.
\end{proof}

Summing up, we know $[\It\Txt\Ex] \subsetneq [\Txt\Ex]\perp[\It\Inf\Ex] \subsetneq[\Inf\Ex]$.

\bigskip
In the following we give a procedure to generate more separating classes in $[\Txt\Ex]\setminus[\It\Inf\Ex]$.
With the help of the Boolean function $\mathbf{f}$ being defined in Definition~\ref{BoolMap}
we obtain from an indexable family $\CalL\in[\Inf\Ex]\setminus[\It\Inf\Ex]$ an indexable family $\mathbf{f}(\CalL)\in[\Txt\Ex]\setminus[\It\Inf\Ex]$.

\medskip
The idea is to apply the Boolean function $\mathbf{f}$, defined in the following, to an indexable family, a set of informants and to a hypothesis space being a candidate to witness the learnability.
With this notation we can draw conclusions from the learnability in the setting before applying $\mathbf{f}$ to the setting after applying $\mathbf{f}$ and vice versa.

\begin{definition}\label{BoolMap}
We refer to the function $\mathbf{f}\colon \mathcal{P}(\mathbb{N}) \to \mathcal{P}(\mathbb{N})$ defined by
\[ \left( 2n \in \mathbf{f}(L) \Leftrightarrow n \in L \right) \wedge \left( 2n+1 \in \mathbf{f}(L) \Leftrightarrow n \not\in L \right) \]
as the Boolean mapping.
For a set of languages $\CalL$ we define $\mathbf{f}(\mathcal{L}) = \{\mathbf{f}(L) | L \in \mathcal{L}\}$.
\end{definition}

Note that for an indexable class $\mathcal{L}$ the image $f(\mathcal{L})$ is again an indexable class.

To obtain a result also applicable in other contexts, we generalize the notation.
Let $\mathcal{I}$ be a set of informants (texts), for example the ones containing each information only once or infinitely often.
$M$ \emph{learns $L$ from $\mathcal{I}$} if it is successful on every $I\in\mathcal{I}$ for $L$.
$M$ \emph{learns $\CalL$ from $\mathcal{I}$} if it learns every $L\in\CalL$ from $\mathcal{I}$.
We denote the collection of all $\CalL$ learnable from $\mathcal{I}$ by $[\mathcal{I}\Ex]$.


The idea is to apply the Boolean function $\mathbf{f}$ to an indexable family, a set of informants and a hypothesis space possibly witnessing the learnability.
With this notation we can draw conclusions from the learnability in the setting before applying $\mathbf{f}$ to the setting after applying $\mathbf{f}$ and vice versa.

\begin{definition}
We refer to the function $\mathbf{f}\colon \mathcal{P}(\mathbb{N}) \to \mathcal{P}(\mathbb{N})$ defined by
\[ \left( 2n \in \mathbf{f}(L) \Leftrightarrow n \in L \right) \wedge \left( 2n+1 \in \mathbf{f}(L) \Leftrightarrow n \not\in L \right) \]
as the Boolean mapping.
For a set of languages $\CalL$ we define $\mathbf{f}(\mathcal{L}) = \{\mathbf{f}(L) | L \in \mathcal{L}\}$.
For an informant $I$ for $L$ we obtain an informant $\mathbf{f}(I)$ for $\mathbf{f}(L)$ by interweaving $I_+$ and $I_-$ where
$$I_+(t) = \begin{cases}
(2n_t,1) &\text{if } I(t)=(n_t,1); \\
(2n_t+1,1) &\text{if } I(t)=(n_t,0).
\end{cases} \quad
\text{and}
\quad
I_-(t) = \begin{cases}
(2n_t+1,0) &\text{if } I(t)=(n_t,1); \\
(2n_t,0) &\text{if } I(t)=(n_t,0).
\end{cases}$$
Moreover, the projection of $I_+$ to the first coordinate yields a text for $\mathbf{f}(L)$.
For a set of informants $\mathcal{I}$ we define the corresponding sets of informants $\mathbf{f}(\mathcal{I})$ and texts $T_{\mathbf{f}}(\mathcal{I})$ by
\[ \mathbf{f}(\mathcal{I}) := \{ \mathbf{f}(I) \mid  I \in \mathcal{I}\} \quad \text{and} \quad T_{\mathbf{f}}(\mathcal{I}) := \{ \pr_1\circ I_+ \mid  I \in \mathcal{I}\}. \]
\end{definition}

Note that for an indexable class $\mathcal{L}$ the image $f(\mathcal{L})$ is again an indexable class.

\bigskip
We will apply the following result to the full set of informants but state it more generally for arbitrary sets of informants $\CalI$.

\begin{theorem} \label{InfToTxt}
Let $\mathcal{I}$ be a set of informants, $\mathcal{L} \subseteq \{\text{pos}(I) | I \in \mathcal{I}\}$ a concept class and $\CalH$ an indexable family as suitable fixed hypothesis space.
Consider the Boolean mapping $\mathbf{f}$ from Definition~\ref{BoolMap}.

If $\mathcal{L} \in [\mathcal{I}\Ex_\CalH]$, 
then $\mathbf{f}(\mathcal{L}) \in [T_{\mathbf{f}}(\mathcal{I})\Ex_{\mathbf{f}(\CalH)}]$.

Moreover, if $\mathcal{I}$ is upwards closed with respect to the subsequence relation,
then $\mathcal{L} \in [(\It)\mathcal{I} \Ex_\CalH]$ is equivalent to $\mathbf{f}(\mathcal{L}) \in [(\It)\mathbf{f}(\mathcal{I}) \Ex_{\mathbf{f}(\CalH)}]$.
\begin{proof}
Let $\mathbf{f}$, $\mathcal{I}$, $\mathcal{L}$ and $\CalH$ be as stated above.

$\mathcal{L} \in \mathbf{[\mathcal{I}Ex_\CalH]} \Rightarrow \mathbf{f}(\mathcal{L}) \in [T_{\mathbf{f}}(\mathcal{I})\Ex_{\mathbf{f}(\CalH)}]:$
Let $M$ be a learner for $\mathcal{L}$ from $\mathcal{I}$.
Let $\mathbf{f}(L)\in \mathbf{f}(\CalL)$ and $T\in T_{\mathbf{t}}(\mathcal{I})$ a text for $\mathbf{f}(L)$.
Then there is an informant $I\in\mathcal{I}$ for $L$ such that $T=\pr_1\circ I_+$.
If for every $t\in\N$ we denote the first and second coordinate of $I(t)$ by $n_t$ and $\lambda_t$, respectively, 
we obtain $T=(2n_t +1-\lambda_t)_{t\in\N}$.
Therefore, we can in a computable way reconstruct $I[t]$ from $T[t]$.
We define a learner $M'$ which simulates $M$ by $M'(T[t])=M(I[t])$.
It is easy to see that $M'$ learns $\mathbf{f}(\mathcal{L})$ from $T_{\mathbf{f}}(\mathcal{I})$.

If $\mathcal{I}$ is upwards closed with respect to the subsequence relation, 
$\mathcal{L} \in \mathbf{[\It \mathcal{I} Ex_\CalH]} \Rightarrow f(\mathcal{L}) \in \mathbf{[\It f(\mathcal{I}) Ex_{\mathbf{f}(\CalH)}]}:$
The proof is very similar to the last paragraph.
Let $M$ be a learner for $\mathcal{L}$ from $\mathcal{I}$.
Let $\mathbf{f}(L)\in \mathbf{f}(\CalL)$ and $I'\in \mathbf{f}(\mathcal{I})$ an informant for $\mathbf{f}(L)$.
Then there is an informant $I\in\mathcal{I}$ for $L$ such that $I'$ results from interweaving $I_+$ and $I_-$.
We compute $\tilde{I}(t)=(\lfloor \frac{x_t}{2} \rfloor, (x_t-w_t) \mod 2)$ from $I'(t)=(x_t,w_t)$ and define $M'$ by $M'(I'[t])=M(\tilde{I}[t])$.
Because $\tilde{I}$ contains $I$ as a subsequence, we obtain $\tilde{I}\in\mathcal{I}$.
Again, it is easily verified that $M'$ learns $\mathbf{f}(\mathcal{L})$ from $\mathbf{f}(\mathcal{I})$.
Moreover, it easy to see that $M'$ is iterative, in case $M$ is.

$f(\mathcal{L}) \in \mathbf{[\It f(\mathcal{I}) \Ex_{\mathbf{f}(\CalH)}]} \Rightarrow \mathcal{L} \in \mathbf{[\It \mathcal{I} Ex_\CalH]}$:
We proceed in a similar fashion.
Let $M'$ be a learner for $\mathbf{f}(\CalL)$ from $\mathbf{f}(\mathcal{I})$.
Let $L\in\CalL$ and $I$ an informant for $L$.
We recursively construct initial segments $\sigma_t$ with $|\sigma_t|=2t$ for the informant $\mathbf{f}(I)$ for $\mathbf{f}(L)$ from $I$ as follows: 
$\sigma_0=\emptyset$;
if $\sigma_t$ is defined and $I(t)=(n_t,\lambda_t)$ then let $\sigma_{t+1}=\sigma_t(2n_t+1-\lambda_t,1)(2n_t+\lambda_t,0)$.
Clearly, $\mathbf{f}(I)=\bigcup_{t\in\N} \sigma_t$.
The learner $M(I[t])=M'(\sigma_t)$ learns $\CalL$ from $\mathcal{I}$.
Finally, if $M'$ is iterative, so is $M$.
\end{proof}
\end{theorem}

If $\mathcal{I}$ is the set of all informants for $\CalL$, then $T_{\mathbf{f}}(\mathcal{I})$ is the set of all texts for $\mathbf{f}(\CalL)$. $\mathbf{f}(\mathcal{I})$ is the set of all informants for $\mathbf{f}(\CalL)$ that have the positive and negative informations in the order given by interweaving.

\begin{corollary}
Consider the Boolean mapping $\mathbf{f}$ from Definition~\ref{BoolMap}.
Then for indexable concept classes and hypothesis spaces holds: $\mathcal{L} \in \mathbf{[InfEx]} \Rightarrow \mathbf{f}(\mathcal{L}) \in \mathbf{[TxtEx]}$, \linebreak[3]
and $\mathcal{L} \in \mathbf{[\It Inf Ex]} \Leftarrow \mathbf{f}(\mathcal{L}) \in \mathbf{[\It Inf Ex]}$.
\begin{proof}
For the second implication note that $\mathbf{f}(\mathcal{L}) \in \mathbf{[\It Inf Ex]} \Rightarrow \mathbf{f}(\mathcal{L}) \in \mathbf{[\It \,f(Inf) Ex]} \Rightarrow \mathcal{L} \in \mathbf{[\It Inf Ex]}$.
\end{proof}
\end{corollary}

Therefore, every set of languages separating $[\It\Inf\Ex]$ and $[\Inf\Ex]$ yields a separating class for $[\It\Inf\Ex]$ and $[\Txt\Ex]$.

\begin{corollary} \label{TxtItInf}
Consider the Boolean mapping $\mathbf{f}$ from Definition~\ref{BoolMap}.
Let $\CalL$ be an indexable concept class and require that learnability is witnessed by indexable hypothesis spaces.
Then $\mathcal{L} \in \mathbf{[InfEx]}$ implies $\mathbf{f}(\mathcal{L}) \in \mathbf{[TxtEx]}$.
Moreover, from $\mathbf{f}(\mathcal{L}) \in \mathbf{[\It Inf Ex]}$ we can conclude $\mathcal{L} \in \mathbf{[\It Inf Ex]}$.
\end{corollary}


\section{Total and Canny Learners} \label{TotCan}

For the rest of this section, without further notation, all results are understood with respect to the $W$-hypothesis space defined in the following.
We fix a programming system $\varphi$ as introduced in \cite{Roy-Cas:b:94}. Briefly, in the $\varphi$-system, for a natural number $p$, we denote by $\varphi_p$ the partial computable function with program code $p$.
We also call $p$ an \emph{index} for $W_p$ defined as $\dom(\varphi_p)$.
In reference to a Blum complexity measure, for all $p, t \in \N$, we denote by $W^t_p \subseteq W_p$ the recursive set of all natural numbers less or equal to $t$, on which the machine executing $p$ halts in at most $t$ steps.
Moreover, by s-m-n we refer to a well-known recursion theoretic observation, which gives nice finite and infinite recursion theorems, like Case's Operator Recursion Theorem $\ORT$.

\medskip
Let us discuss Theorem~\ref{InfToTxt} for $W$-indices. For, let $p$ be such that $W_p\in\CalL$.
There is an obvious mapping from an $W$-index $q$ for $\mathbf{f}(W_p)\in\mathbf{f}(\CalL)$ to some $p'$ with $W_p=W_{p'}$.
Unfortunately, it is not possible to map a $W$-index for a non-recursive $W_p$ to a $W$-index for $\mathbf{f}(W_p).$

\medskip
The question whether excluding partial functions as learners, denoted by $\mathcal{R}$, makes some sets of languages unlearnable has been investigated.
Allowing only total learners does not restrict full-information learning from informant and text, i.e. $[\mathcal{R}\Inf\Ex]=[\Inf\Ex]$ and $[\mathcal{R}\Txt\Ex]=[\Txt\Ex]$.
On the other hand \cite{CaseMoelius2009} showed $[\mathcal{R}\It\Txt\Ex]\subsetneq[\It\Txt\Ex]$.

We show that totality restricts iterative learning from informant. 

\begin{theorem}\label{thm:TotalityRestricts}
$[\It\Inf\Ex]\setminus[\CalR\It\Inf\Ex]\neq\varnothing$.
\end{theorem}
\begin{proof}
Let $o$ be an index for $\varnothing$ and define the iterative learner $M$ for all $\xi \in \Nttwo$ by
\begin{align*}
M(\varnothing) &= o; \\
h_{M}(h,\xi) &= \begin{cases}
\varphi_{\pr_1(\xi)}(0), &\text{else if } \pr_2(\xi)=1 \text{ and } h \notin \ran(\ind); \\
h, &\text{otherwise.}
\end{cases}
\end{align*}
We argue that $\CalL := \{\, L \subseteq \N \mid 
L\in\It\Inf\Ex(M) \,\}$ is not learnable by a total learner from informants.
Assume towards a contradiction $M'$ is such a learner.

For a finite informant sequence $\sigma$ we denote by $\overline{\sigma}$ the corresponding canonical finite informant sequence, ending with $\sigma$'s datum with highest first coordinate.
Then by padded ORT there are $e \in \N$ and a strictly increasing computable function $a: \Seq{\N} \to \N$, such that for all $\sigma \in \Seq{\N}$ and all $i \in \N$
\begin{align}
\sigma_0 &= \varnothing; \nonumber \\
\sigma_{i+1} &= \sigma_i\,\concat \begin{cases}
(a(\sigma_i),1), &\text{if } M'(\overline{\sigma_i\:\!\concat (a(\sigma_i),1)})\neq M'(\overline{\sigma_i}); \\
\varnothing, &\text{otherwise;}
\end{cases} \label{eq:MCNewLearner} \\
W_{e} &= \bigcup_{i \in \N} \ps(\overline{\sigma_i}); \nonumber \\
\varphi_{a(\sigma)}(x)&=\begin{cases}
e, & \text{if } M'(\overline{\sigma\concat (a(\sigma),1)})\neq  M'(\overline{\sigma}); \\
\ind_{\ps(\sigma)\cup\{a(\sigma)\}}, & \text{otherwise;}
\end{cases} \nonumber
\end{align}
Clearly, we have $W_e \in \CalL$ and thus $M'$ also $\Inf\Lim$-learns $W_e$.
By the $\Lim$-convergence there are $e', t_0 \in \N$, where $t_0$ is minimal, such that $W_{e'}=W_{e}$ and for all $t \geq t_0$ we have $M'(\bigcup_{i\in\N}\overline{\sigma_i}[t])=e'$ and hence by \eqref{eq:MCNewLearner} for all $i$ with $|\overline{\sigma_i}|\geq t_0$
$$M'(\overline{\sigma_i\:\!\concat (a(\sigma_i),1)})= M'(\overline{\sigma_i})=M'(\overline{\sigma_i\:\!\concat (a(\sigma_i),0)}).$$

It is easy to see, that $W_e=\ps(\sigma_i)$ and $W_e\cup\{a(\sigma_i)\} \in \CalL$.
On the other hand $M'$ is iterative and hence does not learn $W_e$ and $W_e\cup\{a(\sigma_i)\}$.
\end{proof}

The following definition is central in investigating the learning power of iterative learning from texts, see \cite{Cas-Moe:c:07:nonuit} and \cite{jain2016role}. We transfer it to learning from informants.

\begin{definition}
A learner $M$ from informant is called \emph{canny} in case for every finite informant sequence $\sigma$ holds
\begin{enumerate}
\item if $M(\sigma)$ is defined then $M(\sigma)\in\N$;
\item for every $x\in\N\setminus\cnt(\sigma)$ and $i\in\{0,1\}$ a mind change $M(\sigma\concat(x,i))\neq M(\sigma)$ implies for all finite informant sequences $\tau$ with $\sigma\concat(x,i)\inseg\tau$ that $M(\tau\concat(x,i))= M(\tau)$.
\end{enumerate}
\end{definition}

Hence, the learner is canny in case it always outputs a hypotheses and no datum twice causes a mind change of the learner.
Also for learning from informant, the learner can be assumed canny.

\begin{lemma} \label{ItCannyInf}
For every iterative learner $M$, there exists a canny iterative learner $N$ such that $$\Inf\Ex(M)\subseteq \Inf\Ex(N).$$
\end{lemma}
\begin{proof}
Let $f$ be a computable 1-1 function mapping every finite informant sequence $\sigma$ to a natural number encoding a program with $W_{f(\sigma)}=W_{M(\sigma)}$ if $M(\sigma)\in\N$ and $W_{f(\sigma)}=\varnothing$ otherwise.
Clearly, $\sigma$ can be reconstructed from $f(\sigma)$.
We define the canny learner $M'$ by letting
\begin{align*}
M'(\varnothing)&=f(\varnothing) \\
h_{M'}(f(\sigma),(x,i))&=
\begin{cases}
f(\sigma\concat (x,i)), &\text{if } x \notin \ps(\sigma)\cup\ng(\sigma) \wedge M(\sigma\concat (x,i))\!\downarrow \:\neq M(\sigma)\!\downarrow; \\
f(\sigma), &\text{if } M(\sigma\concat (x,i))\!\downarrow \:= M(\sigma)\!\downarrow \vee \; x \in \cnt(\sigma); \\
\uparrow, &\text{otherwise.}
\end{cases}
\end{align*}
$M'$ mimics $M$ via $f$ on a possibly finite informant subsequence of the originally presented informant with ignoring data not causing mind changes of $M$ or that has already caused a mind change.

Let $L\in\Inf\Ex(M)$ and $I'\in\Inf(L)$.
As $M$ has to learn $L$ from every informant for it, $M'$ will always be defined.
Further, let $\sigma_0=\varnothing$ and 
\begin{align*}
\sigma_{t+1}&=
\begin{cases}
\sigma_t\concat I'(t), &\text{if } I'(t) \notin \ran(\sigma_t) \wedge M(\sigma_t\concat I'(t))\!\downarrow \:\neq M(\sigma_t)\!\downarrow; \\
\sigma_t, &\text{otherwise.}
\end{cases}
\end{align*}
Then by induction for all $t\in\N$ holds $M'(I'[t])=f(\sigma_t)$.

The following function translates between the two settings
\begin{align*}
\nonsimu(0) &= 0; \\
\nonsimu(t+1) &= \min\{ r>\nonsimu(t) \mid I'(r-1)\notin \ran(\sigma_{\nonsimu(t)}) \}.
\end{align*}
Intuitively, the infinite range of $\nonsimu$ captures all points in time $r$ at which a datum that has not caused a mind change so far, is seen and a mind-change of $M'$ is possible.
Thus the mind change condition is of interest in order to decide whether $\sigma_{\nonsimu(t+1)}\neq\sigma_{\nonsimu(t)}$.
Note that $\sigma_r=\sigma_{\nonsimu(t)}$ for all $r$ with $\nonsimu(t)\leq r <\nonsimu(t+1)$.

Let $I(t)=I'(\nonsimu(t+1)-1)$ for all $t\in\N$.
Since only already observed data is ommited, $I$ is an informant for $L$.

We next argue that $M(I[t])=M(\sigma_{\nonsimu(t)})$ for all $t\in\N$.
As $I[0]=\varnothing=\sigma_0$, the claim holds for $t=0$.
Now we assume $M(I[t])=M(\sigma_{\nonsimu(t)})$ and show $M(I[t+1])=M(\sigma_{\nonsimu(t+1)})$ as follows
\begin{align*}
M(I[t+1]) = M(I[t]\concat I(t))
= M(\sigma_{\nonsimu(t)}\concat I(t)).
\end{align*}
As by the definitions of $I$ and $\nonsimu$ we have $I(t)=I'(\nonsimu(t+1)-1)\notin\ran(\sigma_{\nonsimu(t)})$ there are two cases:
\begin{enumerate}
\item If $M(\sigma_{\nonsimu(t)}\concat I(t))= M(\sigma_{\nonsimu(t)})$, then from $\sigma_{\nonsimu(t+1)-1}=\sigma_{\nonsimu(t)}$ and the definition of $M'$ we obtain $\sigma_{\nonsimu(t+1)}=\sigma_{\nonsimu(t)}$.
Putting both together the claimed equality $M(\sigma_{\nonsimu(t)}\concat I(t))=M(\sigma_{\nonsimu(t+1)})$ follows.
\item If $M(\sigma_{\nonsimu(t)}\concat I(t))\neq M(\sigma_{\nonsimu(t)})$, the definition of $M'$ yields $\sigma_{\nonsimu(t+1)}=\sigma_{\nonsimu(t)}\concat I(t)$. Hence the claimed equality also holds in this case.
\end{enumerate}

We now argue that $M'$ explanatory learns $L$ from $I'$.
In order to see this, first observe $\sigma_{\nonsimu(t+1)}=\sigma_{\nonsimu(t)}$ if and only if $M(I'[t+1])=M(I'[t])$ for every $t\in\N$.
This is because
\begin{align*}
\sigma_{\nonsimu(t+1)}=\sigma_{\nonsimu(t)} 
&\Leftrightarrow M(\,\sigma_{\nonsimu(t)}\concat I(t)\,) = M(\sigma_{\nonsimu(t)}) \\
&\Leftrightarrow M(I[t]\concat I(t)) = M(I[t]) \\
&\Leftrightarrow M(I[t+1]) = M(I[t]).
\end{align*}

As $I$ is an informant for $L$, the learner $M$ explanatory learns $L$ from $I$.
Hence there exists some $t_0$ such that $W_{M(I[t_0])}=L$ and for all $t\geq t_0$ holds $M(I[t])=M(I[t_0])$.
With this follows $\sigma_{\nonsimu(t)}=\sigma_{\nonsimu(t_0)}$ for all $t\geq t_0$.
As for every $r$ there exists some $t$ with $\nonsimu(t)\leq r$ and $\sigma_r=\sigma_{\nonsimu(t)}$, we obtain $\sigma_r = \sigma_{\nonsimu(t_0)}$ for all $r\geq \nonsimu(t_0)$.
We conclude $M'(I'[t])=f(\sigma_t)=f(\sigma_{\nonsimu(t_0)})$ for all $t \geq \nonsimu(t_0)$ and by the definition of $f$ finally $W_{f(\sigma_{\nonsimu(t_0)})}=W_{M(\sigma_{\nonsimu(t_0)})}=W_{M(I[t_0])}=L$.
\end{proof}

\section{Additional Requirements}\label{AddRequ}

In the following we review additional properties one might require the learning process to have in order to consider it successful.
For this, we employ the following notion of consistency.

As in \cite{lange2008learning} according to \cite{Blu-Blu:j:75} and  \cite{Bar:c:77:aut-func-prog} for $A \subseteq \N$ we define
\begin{align*}
\Comp(f,A) \quad &:\Leftrightarrow \quad \ps(f) \subseteq A \;\wedge\; \ng(f) \subseteq \N\setminus A
\end{align*}
and say \emph{$f$ is consistent with $A$} or \emph{$f$ is compatible with $A$}.

\medskip
Learning restrictions incorporate certain desired properties of the learners' behavior relative to the information being presented.
We state the definitions for learning from informant here.

\begin{definition}
\label{def:LearningRestrictions}
Let $M$ be a learner and $I$ an informant.
We denote by $h_t=M(I[t])$ the hypothesis of $M$ after observing $I[t]$ and write
\begin{enumerate}
        \item $\Conv(M,I)$ (\cite{angluin1980inductive}), if $M$ is \emph{conservative on $I$}, i.e., for all $s, t$ with $s \leq t$ the consistency 
				$\Comp(I[t],W_{h_s})$ implies $h_s = h_t.$
        \item  $\Dec(M,I)$ (\cite{Osh-Sto-Wei:j:82:strategies}),
        if $M$ is \emph{decisive on $I$}, i.e.,
        for all $r, s, t$ with $r \leq s \leq t$ the semantic equivalence \\
        $W_{h_r} = W_{h_t}$ implies the semantic equivalence $W_{h_r} = W_{h_s}.$
        \item $\Caut(M,I)$ (\cite{STL1}),
				if $M$ is \emph{cautious on $I$}, i.e.,
				for all $s, t$ with $s \leq t$ holds
        $\neg W_{h_t} \subsetneq W_{h_s}.$
				\item $\WMon(M,I)$ (\cite{j-mniifp-91},\cite{Wie:c:91}), if $M$ is \emph{weakly monotonic on $I$}, i.e., for all $s, t$ with $s \leq t$ holds \\
        $\Comp(I[t], W_{h_s})
        \;\Rightarrow\; W_{h_s} \subseteq W_{h_t}.$
        \item $\Mon(M,I)$ (\cite{j-mniifp-91},\cite{Wie:c:91}), if $M$ is \emph{monotonic on $I$}, i.e., for all $s, t$ with $s \leq t$ holds \\
        $W_{h_s} \cap \ps(I) \subseteq W_{h_t}\cap\ps(I).$
        \item  $\SMon(M,I)$ (\cite{j-mniifp-91},\cite{Wie:c:91}), if $M$ is \emph{strongly monotonic on $I$}, i.e., for all $s, t$ with $s \leq t$ holds
        $W_{h_s} \subseteq W_{h_t}.$
        \item  $\NU(M,I)$ (\cite{Bal-Cas-Mer-Ste-Wie:j:08}), if $M$ is \emph{non-U-shaped on $I$}, i.e., for all $r, s, t$ with $r \leq s \leq t$ 
				the semantic success $W_{h_r} = W_{h_t} = \ps(I)$ implies the semantic equivalence $W_{h_r} = W_{h_s}.$
        \item  $\SNU(M,I)$ (\cite{Cas-Moe:j:11:optLan}), if $M$ is \emph{strongly non-U-shaped on $I$}, i.e., for all $r, s, t$ with $r \leq s \leq t$
				the semantic success $W_{h_r} = W_{h_t} = \ps(I)$ implies the syntactic equality $h_r = h_s.$
        \item  $\SDec(M,I)$ (\cite{kotzing2014map}), if $M$ is \emph{strongly decisive on $I$}, i.e., for all $r, s, t$ with $r \leq s \leq t$
				the semantic equivalence $W_{h_r} = W_{h_t}$ implies the syntactic equality $h_r = h_s.$
\end{enumerate}%
\end{definition}

It is easy to observe that $\Conv(M,I)$ implies $\SNU(M,I)$ and $\WMon(M,I)$;
$\SDec(M,I)$ implies $\Dec(M,I)$ and $\SNU(M,I)$;
$\SMon(M,I)$ implies $\Caut(M,I), \Dec(M,I), \Mon(M,I)$, $\WMon(M,I)$
and finally $\Dec(M,I)$ and $\SNU(M,I)$ imply $\NU(M,I)$.

\medskip
The text variants can be found in \cite{jain2016role} where all pairwise relations $=$, $\subsetneq$ or $\perp$ between the sets $[\It\Txt\delta\Ex]$ (iterative learners from text) for $\delta\in\Delta$, where $\Delta=\{\Conv,\Dec,\Caut,\WMon,\Mon,\SMon,$ $\NU,\SNU,\SDec\}$, are depicted.
The complete map of all pairwise relations between the sets $[\Inf\delta\Ex]$ (full-information learners from informant) for $\delta\in\Delta$ can be found in \cite{As-Koe-Sei2018_informants}.
For iterative learning from informants this complete map is not known.
We sum up the current status in the following.

\bigskip
Recall the indexable family $\CalL =\{2\N\}\cup \{L_k,L'_k \mid k\in\N \}$ with $L_k= 2\N\cup\{2k+1\}$ and $L'_k = L_k\setminus \{2k\}$, separating $[\It\Txt\Ex]$ from $[\Txt\Ex]$.
Clearly, $\CalL \in [\totalCp\It\Inf\Conv\SDec\Mon\Ex]$.
%
%
With a locking sequence argument we can observe $[\It\Inf\SMon\Ex]\subsetneq[\It\Inf\delta\Ex]$ for all $\delta\in\Delta\setminus\{\SMon\}$.

\bigskip
If we denote by $\Inf_\can$ the set of all informants labelling the natural numbers according to their canonical order, we obtain $\Fin\cup\{\N\}\in[\totalCp\It\Inf_{\can}\Cons\Conv\SDec\Mon\Ex]$ and thus in contrast to full-information learning from informant $[\It\Inf_{\can}\Ex]\neq[\It\Inf\Ex]$, see \cite{As-Koe-Sei2018_informants}.

\bigskip
Theorem~\ref{thm:TotalityRestricts} can be restated as.
\begin{theorem}
$[\It\Inf\Conv\SDec\SMon\Ex]\setminus[\CalR\It\Inf\Ex]\neq\varnothing$.
\end{theorem}

\bigskip
It has been observed that requiring a monotonic behavior of the learner is restrictive.

\begin{theorem}\cite{lz-tmllc-92}
\label{Mon}
There exists an indexable family in 
$[\It\Inf\Mon\Ex]\subsetneq[\It\Inf\Ex]$.
\end{theorem}


\bigskip
It is easy to see that requiring a cautious behavior of the learner is also restrictive.

\begin{theorem} \label{ItInfCautEx}
There exists an indexable family in 
$[\It\Inf\Caut\Ex]\subsetneq[\It\Inf\Ex]$.
\begin{proof}
The indexable family $\{\N\}\cup\{\N\setminus\{x\} \mid x\in\N\}$ is clearly not cautiously learnable but conservatively, strongly decisively and monotonically learnable by a total iterative learner from informant.
\end{proof}
\end{theorem}

\begin{corollary}
$[\It\Inf\Caut\Ex]\perp[\It\Inf\Mon\Ex]$
\end{corollary}\bigskip

%
%

Moreover, requiring a conservative learning behavior is also restrictive.

\begin{theorem}\cite{Jai-Lan-Zil:j:07}
\label{Conv}
There exists an indexable family in 
$[\It\Inf\Conv\Ex]\subsetneq[\It\Inf\Ex]$.
\end{theorem}

Indeed, they provide an indexable family in $[\It\Inf\Caut\WMon\NU\Dec\Ex]\setminus[\It\Inf\Conv\Ex]$ and
an indexable family in $[\totalCp\It\Txt\Caut\Conv\SDec\Ex]\setminus[\It\Inf\Mon\Ex]$.

\bigskip
Hence the map differs from the map on iterative learning from text in \cite{jain2016role} as $\Caut$ is restrictive and also from the map of full-information learning in 
\cite{As-Koe-Sei2018_informants} from informant as $\Conv$ is restrictive too.
It has been open how $\WMon$, $\Dec$, $\NU$, $\SDec$ and $\SNU$ relate to each other and the other requirements.
We show that also $\SNU$ restricts $\It\Inf\Ex$ with an intricate $\ORT$-argument.

\begin{theorem} \label{SNUREstrictsItInfEx}
$[\It\Inf\SNU\Ex]\subsetneq[\It\Inf\Ex]$
\end{theorem}
\begin{proof}
Let $M$ be a learner as follows, where the initial hypothesis is $o$, an index for $\emptyset$. We consider input data $x$ with given label $\ell \in \{0,1\}$. 
$$
\forall e,x,\ell: h_M(e,(x,\ell)) = \begin{cases}
e,											&\mbox{if }e = o \wedge \ell=0;\\
\pad(\varphi_{x}(0),x),				&\mbox{else if }e = o \wedge \ell=1;\\
\pad(\varphi_{y}(\langle e',x,\ell\rangle),y),		&\mbox{else, with }e = \pad(e',y).
\end{cases}
$$

Let $\CalL$ be what $M$ learns and suppose $M'$ learns $\CalL$ also SNU.

We define strictly increasing computable functions $a, b, e_1, e_2: \N \to \N$ and $e_0\in\N$ by ORT.
Thereby, we interpret $a$ and $b$ as data streams and for all $k,t$ the numbers $e_0$, $e_1(\langle k,t\rangle)$ and $e_2(\langle k,t\rangle)$ as hypotheses.
We start with defining $a$ and $b$ by letting for all $i, k \in\N$
\begin{align*}
\varphi_{a(i)}(z)&=\begin{cases}
e_1(\langle k,k \rangle), &\text{if } z=\langle e_0,b(k),1 \rangle;\\
e_0, &\text{else if } z=0 \vee z=\langle e_0,x,\ell \rangle;\\
e_1(\langle k,t \rangle), &\text{else if } z=\langle e_1(\langle k, s\rangle),a(t),1 \rangle \wedge t\geq s \wedge W_{e_0}^t[k]\neq W_{e_0}^s[k];\\
e_2(\langle k, k\rangle), &\text{else if } z=\langle e_1(\langle k, s\rangle),a(t),0 \rangle \wedge t\geq k;\\
e_2(\langle k,t \rangle), &\text{else if } z=\langle e_2(\langle k, s\rangle),a(t),\ell \rangle \wedge t\geq s \wedge W_{e_0}^t[k]\neq W_{e_0}^s[k];\\
e, &\text{else if } z=\langle e,x,\ell \rangle;
\end{cases}\\
\varphi_{b(k)}(z)&=\begin{cases}
e_1(\langle k,k\rangle), &\text{if } z=0;\\
e_1(\langle k,t \rangle), &\text{else if } z=\langle e_1(\langle k, s\rangle),a(t),1 \rangle \wedge t\geq s \wedge W_{e_0}^t[k]\neq W_{e_0}^s[k];\\
e_2(\langle k, k\rangle), &\text{else if } z=\langle e_1(\langle k, s\rangle),a(t),0 \rangle \wedge t\geq k;\\
e_2(\langle k,t \rangle), &\text{else if } z=\langle e_2(\langle k, s\rangle),a(t),\ell \rangle \wedge t\geq s \wedge W_{e_0}^t[k]\neq W_{e_0}^s[k];\\
e, &\text{else if } z=\langle e,x,\ell \rangle;
\end{cases}
\end{align*}
Before we define $W_{e_0}$, $W_{e_1(\langle k,t\rangle)}$ and $W_{e_2(\langle k,t\rangle)}$, note that, while $M$ sees only negatively labeled data, it sticks to $o$ as hypothesis.
Once a positive $a$-datum is seen, it sticks to $e_0$ as hypothesis.
The first positive $b(k)$-datum makes it change its mind to $e_1(\langle k,k\rangle)$.
Any \emph{negative} $a$-datum after the positive $b(k)$-datum leads to $e_2(\langle k,k\rangle)$.
As the second coordinate in $\langle k,t\rangle$ will tell us which canonical informant sequence $W_{e_0}^t[k]$ we consider, we enlarge it whenever neccessary in order to guarantee $W_{e_0}^t[k]=W_{e_0}[k]$ in the limit.

We give the definitions of what to list into $W_{e_0}$, $W_{e_1(\langle k,t\rangle)}$ and $W_{e_2(\langle k,t\rangle)}$ as algorithms.

In $W_{e_0}$ we enumerate all $a(i)$ on which $M'$ changes its mind when labeled positively while $M'$ observes the canonical informant for $W_{e_0}$.
For convenience, in the definition of $W_{e_0}$ we let $a(-1)=-1$ and denote by $[u,w]$ the set of all integers $v$ with $u\leq v\leq w$.

\begin{algorithm2e}
	$e$ $\assign$ initial hypothesis of $M'$\;
	\For{$i = 0$ \textrm{\bf to} $\infty$}{
		\If{$h^\ast_{M'}(e,[a(i-1)+1,a(i)-1]\times\{0\}\concat(a(i),1))\,\downarrow \neq e$}{
			$e$ $\assign$ $h_{M'}(e,[a(i-1)+1,a(i)-1]\times\{0\}\concat(a(i),1))$\;
			list $a(i)$ into $W_{e_0}$\;
		}
		\ElseIf{$h^\ast_{M'}(e,[a(i-1)+1,a(i)-1]\times\{0\}\concat(a(i),0))\,\downarrow \neq e$}{
			$e$ $\assign$ $h_{M'}(e,[a(i-1)+1,a(i)-1]\times\{0\}\concat(a(i),0))$\;
		}
	}
	\caption{The definition of $e_0$ in the ORT-argument.}
	\label{alg:eZero}
\end{algorithm2e}

As $M$ learns $W_{e_0}$, also $M'$ has to learn it.
Let $I$ be the canonical informant for $W_{e_0}$ and $k$ be such that $M'(I[i])=M'(I[k])$ for all $i\geq k$ and $W_{M'(I[k])}=W_{e_0}$.

\begin{algorithm2e}
	\textbf{Input:} $\langle k,t\rangle$\;
	$e$ $\assign$ $M'(W_{e_0}^t[k](b(k),1))$\;
	$i$ $\assign$ $k$\;
	list $b(k)$ and the positive information in $W_{e_0}^t[k]$ into $W_{e_1}$ and $W_{e_2}$\;
	\For{$s = 0$ \textrm{\bf to} $\infty$}{
		\While{$h_{M'}(e,(a(i),1)) = e$ \textrm{\bf and} $h_{M'}(e,(a(i),0)) = e$}{%
			list $a(i)$ into $W_{e_1}$\;
			$i$ $\assign$ $i+1$\;
		}
		list all of what is already listed in $W_{e_1}$ into $W_{e_2}$\;
		\If{$h_{M'}(e,(a(i),1)) \neq e$}{
		list $a(i)$ into $W_{e_1}$ and $W_{e_2}$\;
		$e$ $\assign$ $h_{M'}(e,(a(i),1))$\;
		}
		\Else{
		$j$ $\assign$ $i$\;
		$i$ $\assign$ $i+1$\;
		\While{$h_{M'}(e,(a(i),1)) = e$}{%
			list $a(i)$ into $W_{e_1}$ and $W_{e_2}$\;
			$i$ $\assign$ $i+1$\;
		}
		list $a(i)$ into $W_{e_1}$ and $W_{e_2}$\;
		list $a(j)$ into $W_{e_1}$ and $W_{e_2}$\;
		$e$ $\assign$ $h^\ast_{M'}(e,(a(i),1)(a(j),1))$\;
		}
		$i$ $\assign$ $i+1$\;
	}
	\caption{The definition of $e_1(\langle k,t\rangle)$ and $e_2(\langle k,t \rangle)$ in the ORT-argument.}
	\label{alg:eOneTwo}
\end{algorithm2e}

For all $k, t, t'$ with $W_{e_0}^{t}[k]=W_{e_0}^{t'}[k]$ holds $W_{e_1(\langle k,t\rangle)}=W_{e_1(\langle k,t'\rangle)}$ and 
$W_{e_2(\langle k,t\rangle)}=W_{e_2(\langle k,t'\rangle)}$.

\medskip
We will now argue that for $t$ minimal with $W_{e_0}^t[k]=I[k]$ every possible outcome of Algorithm 2 is contradictory.

\begin{enumerate}
	\item If all stages $s$ are visited, then $W_{e_1(\langle k,t\rangle)}=W_{e_2(\langle k,t\rangle)}$ contains essentially all $a(i)$ with $i\geq k$.
	Hence $M$ will eventually output the correct hypothesis $e_1(\langle k,t\rangle)$ while $M'$ makes infinitely many mind changes on a suitable informant $I'$.
	More precisely, the informant $I'$ starts with $I[k](b(k),1)$ and afterwards enumerates all $a(i)$ with $i\geq k$ in the order they were listed into $W_{e_1(\langle k,t\rangle)}$.
	\item If the first while loop does not terminate for some stage $s$, then $W_{e_1(\langle k,t\rangle)}$ and $W_{e_2(\langle k,t\rangle)}$ are different.
	As $W_{e_2(\langle k,t\rangle)}$ is finite, $M$ learns it by changing its mind on some negative $a$-datum.
	On the other hand $W_{e_1(\langle k,t\rangle)}$ contains all $a(i)$ with $i\geq k$ and $M$ learns it by not changing its mind.
	Let $e_{s-1}$ denote the current value of variable $e$ when entering the stage $s$.
	By the case assumption, $M'$ does not perform a mind-change on any further positive or negative $a$-datum.
	Therefore, we must have $W_{e_1(\langle k,t\rangle)}=W_{e_{s-1}}=W_{e_2(\langle k,t\rangle)}$, a contradiction.
	\item If the second while loop does not terminate for some stage $s$, then $W_{e_1(\langle k,t\rangle)} = W_{e_2(\langle k,t\rangle)}$ contains all $a(i)$ with $i\geq k$ but $a(j_s)$.
	This is learned by $M$ from any informant (though with different final hypotheses, depending on the informant).
	Again, we let $e_{s-1}$ denote the current value of $e$ when entering stage $s$.
	By the choice of $k$ for all $j\geq k$ holds $M'(I[k]\concat(a(j),1))=M'(I[k])$ and $M'(I[k]\concat(a(j),0))=M'(I[k])$.
	Hence $M'$ on the informant
	$$I'' = I[k](a(j_s),0)(b(k),1)((a(i),1))_{i\geq k, i\neq j_s}$$
	for $W_{e_1(\langle k,t\rangle)}$ outputs $e_{s-1}$ and therefore $e_{s-1}$ must be correct.
	On the other hand $e_{s-1}$	cannot be correct, since $M'$ is SNU and changing its mind on the negative information $(a(j_s),0)$ in the informant
	$$I''' = I[k](b(k),1)((a(i),1))_{i<j_s}(a(j_s),0)((a(i),1))_{i>j_s}$$
	for $W_{e_1(\langle k,t\rangle)}$. 
\end{enumerate}
\end{proof}

We are now attempting to clarify in which sense precisely $\Conv$ is a restriction and more specifically, where exactly and how often there are separations in the implication chains $\Conv \Rightarrow \WMon \Rightarrow \textbf{T}$, $\Conv \Rightarrow \SNU \Rightarrow \NU \Rightarrow \textbf{T}$ and $\SDec \Rightarrow  \Dec \Rightarrow \NU \Rightarrow \textbf{T}$.
In the following we provide a lemma that might help to investigate $\WMon$, $\Dec$ and $\NU$.

\begin{defn}
\label{SemAfsoet}
Denote the set of all unbounded and non-decreasing functions by $\Simu$, i.e., $$\Simu := \{ \,\simu: \N \to \N \mid \forall x \in \N \,\exists t \in \N \colon \simu(t) \geq x \text{ and } \forall t \in \N \colon \simu(t+1) \geq \simu(t) \,\}.$$
Then every $\simu \in \Simu$ is a so called \emph{admissible simulating function}.

\smallskip
A predicate $\beta \subseteq \partialFn \times \CalI$ is
\emph{semantically delayable}, if for all $\simu \in \Simu$, all $I, I' \in \CalI$ and all learners $M, M' \in \partialFn$ holds:
Whenever we have $\ps(I'[t])  \supseteq \ps(I[\simu(t)])$, $\ng(I'[t]) \supseteq \ng(I[\simu(t)])$ and $W_{M'(I'[t])} = W_{M(I[\simu(t)])}$ for all $t \in \N$, from $\beta(M,I)$ we can conclude $\beta(M',I')$.
\end{defn}

\begin{lem}
\label{SemAofsetLemma}
Let $\delta$ be a semantic learning restriction, i.e. $\delta \in \{\Caut,\Dec,\WMon,\Mon,\SMon,\NU \}$.
Then $\delta$ is semantically delayable.
\end{lem}

Lemma~\ref{ItCannyInf} can be generalized as follows.

\smallskip
\begin{lem}~\label{ItCannyInfDelay}
For every iterative learner $M$ and every semantically delayable learning restriction $\delta$, there exists a canny iterative learner $N$ such that $\Inf\delta\Ex(M)\subseteq \Inf\delta\Ex(N)$.
\end{lem}
\begin{proof}
We add $\delta$ in front of $\Ex$ in the proof of Lemma~\ref{ItCannyInf}.
Further, we define a simulating function (Definition~\ref{SemAfsoet}) by
\begin{align*}
\simu(t) &= \max\{ s\in\N \mid \nonsimu(s) \leq t \}.
\end{align*}
It is easy to check that $\simu$ is unbounded and clearly it is non-decreasing.
Then by the definitions of $I$ and $\simu$ we have $\ps(I[\simu(t)]) \subseteq \ps(I'[\nonsimu(\simu(t))])\subseteq \ps(I'[t])$ and similarly $\ng(I[\simu(t)]) \subseteq \ng(I'[t])$ for all $t\in\N$.
As $M'(I'[t]) = f(\sigma_{t})$ and $M(\sigma_{\nonsimu(\simu(t))}) = M(I[\simu(t)])$ for all $t\in\N$, in order to obtain $W_{M'(I'[t])}=W_{M(I[\simu(t)])}$ it suffices to show $W_{f(\sigma_t)}=W_{M(\sigma_{\nonsimu(\simu(t))})}$.
Since $W_{f(\sigma_t)}=W_{M(\sigma_t)}$ for all $t\in\N$, this can be concluded from $\sigma_t=\sigma_{\nonsimu(\simu(t))}$.
But this obviously holds because $\nonsimu(\simu(t))\leq t < \nonsimu(\simu(t)+1)$ follows from the definition of $\simu$.

Finally, from $\delta(M,I)$ we conclude $\delta(M',I')$.
\end{proof}

Two other learning restrictions that might be helpful to understand the syntactic learning criteria $\SNU$, $\SDec$ and $\Conv$ better are the following.

\begin{definition}
Let $M$ be a learner and $I$ an informant.
We denote by $h_t=M(I[t])$ the hypothesis of $M$ after observing $I[t]$ and write
\begin{enumerate}
        \item $\LocConv(M,I)$ (\cite{Jai-Lan-Zil:j:07}), if $M$ is \emph{locally conservative on $I$}, i.e., for all $t$ the mind-change
        $h_t\neq h_{t+1}$ implies $\Comp(I(t), W_{h_t})$.
        \item  $\Wb(M,I)$ (\cite{kotzing2016towards}),
        if $M$ is \emph{witness-based on $I$}, i.e.,
        for all $r, s, t$ with $r < s \leq t$ the mind-change
        $h_r\neq h_s$ implies $\ps(I[s])\cap W_{h_t}\setminus W_{h_r} \neq \varnothing$ $\vee \ng(I[s])\cap W_{h_r}\setminus W_{h_t} \neq \varnothing$.
\end{enumerate}
\end{definition}

Hence, in a locally conservative learning process every mind-change is justified by the datum just seen.
Moreover, a in witness-based learning process each mind-change is witnessed by some false negative or false positive datum.
Obviously, $\LocConv\Rightarrow\Conv$ and $\Wb\Rightarrow\Conv$.

As for learning from text, see \cite{jain2016role}, 
we gain that every concept class locally conservatively learnable by an iterative learner from informant is also learnable in a witness-based fashion by an iterative learner.

\begin{theorem}\label{thrm:LocConvWb}
$[\It\Inf\LocConv\Ex]\subseteq[\It\Inf\Wb\Ex]$
\end{theorem}
\begin{proof}
Let $\CalL$ be a concept class learned by the iterative learner $M$ in a locally conservative manner.
As we are interested in a witness-based learner $N$, we always enlarge the guess of $M$ by all data witnessing a mind-change in the past.
As we want $N$ to be iterative, this is done via padding the set of witnesses to the hypothesis and a total computable function $g$ adding this information to the hypothesis of $M$ as follows:
\begin{align*}
W_{g(\pad(h,\langle MC \rangle))} &= \left( W_h \cup \ps[MC] \right) \setminus \ng[MC]; \\
N(\varnothing)&=g(\pad(M(\varnothing),\langle \varnothing \rangle)); \\
h_N(g(\pad(h,\langle MC \rangle)),\xi)&=
\begin{cases}
g(\pad(h,\langle MC \rangle)), &\text{if } h_M(h,\xi) = h \vee \\
&\hspace{2ex} \xi \in MC;\\
g(\pad(h_M(h,\xi), \\
\hspace{2ex} \langle MC\cup\{\xi\} \rangle)), &\text{otherwise}.
\end{cases}
\end{align*}
Clearly, $N$ is iterative.
Further, whenever $M$ is locked on $h$ and $W_{h}=L$, since $MC$ is consistent with $L$, we also have $W_{g(\pad(f(h),\langle MC \rangle))}=L$.
As $N$ simulates $M$ on an informant omitting all data that already caused a mind-change beforehand, $N$ does explanatory learn $\CalL$.
As $M$ learns locally conservatively and by employing $g$, the learner $N$ acts witness-based.
\end{proof}


\section{Learning Half-Spaces in the Euclidean Plane}\label{sec:halfSpacestwo}

An important concept class for many machine learning algorithms are binary classifiers given by half-spaces.
We will define the language class of halfspaces, show that they from an indexable family and provide a hypothesis space and constructive algorithm making them learnable by an iterative learner from informant. 

\begin{definition}[Coding, Halfspace, $\CalC$]
	For an integer $x\in\Z$ and natural number $i\in\N$ we write $i=\langle x\rangle$ if $i$ is the code of $x$ in the sense of a computable bijection with computable inverse, for example:
	\begin{center}
	\begin{tabular}{*{11}{r}}
		$\Z$\;&$0$&$-1$&$1$&$-2$&$2$&$-3$&$3$&$-4$&$4$&\ldots\\
		$\N$\;&$0$&$1$&$2$&$3$&$4$&$5$&$6$&$7$&$8$&\ldots
	\end{tabular}
	\end{center}
	Moreover, for a computable bijection $\N\times\N\to\N$ with computable inverse, $d>0$ and natural numbers $i, i_0,i_1,\ldots,i_d \in\N$ we write
	\begin{align*}
	i&=\langle i_0,i_1 \rangle, \text{ if } i \text{ is the image of the vector } (i_0,i_1); \\
	i&=\langle i_0,i_1,\ldots,i_d \rangle, \text{ if } i \text{ is the image of the vector } (\langle i_0,\ldots,i_{d-1} \rangle, i_d).
	\end{align*}
	We say that $i$ encodes the vector $(i_0,i_1)$ or $(i_0, i_1, \ldots, i_d)$, respectively.
	
	Let $d>0$. For $a_0,a_1,\ldots, a_d\in\Z$ the corresponding halfspace is given by
	$$H_{\langle \langle a_0\rangle, \langle a_1\rangle,\ldots,\langle a_d\rangle \rangle}
		=\{ \langle\langle x_1\rangle,\ldots,\langle x_d\rangle \rangle \mid a_0 \geq \sum_{i=1}^d a_i x_i \}.$$
	Let $A=\{ \langle \langle a_0 \rangle, 0 \rangle \mid a_0 \in\Z \}$ be the set of all $i$ encoding a vector of integers $(a_0,a_1,\ldots,a_d)$ with $a_1=\ldots=a_d=0$.
	The \emph{concept class of all halfspaces} is defined as $\CalC=\{ H_{i} \mid i\in\N\setminus A\}$.
\end{definition}

\begin{lemma}[\mbox{$\CalC$ is indexable}]
	\label{HalfspacesHypSpace}
		The concept class of halfspaces $\CalC$ is an indexable family.
\end{lemma}
\begin{proof}
		We describe the uniform decision procedure for $\CalC$.
		Given $i$ and $n$ first decode $a_0,a_1,\ldots,a_d,x_1,\ldots,x_d\in\Z$ such that
		$i=\langle \langle a_0\rangle, \langle a_1\rangle,\ldots,\langle a_d\rangle \rangle$ and
		$n=\langle\langle x_1\rangle,\ldots,\langle x_d\rangle \rangle$.
		Then check whether $a_0 \geq a_1 x_1+\ldots+a_d x_d$ and return $1$ if the inequality is true and $0$ otherwise.
\end{proof}

Due to \cite{Gol:j:67} every indexable family is conservatively and consistently learnable by an iterative learner.
Therefore, we immediately obtain.

\begin{corollary}[\mbox{$\CalC\in[\Inf\Ex]$}]
	The concept class of halfspaces $\CalC$ is learnable from informant by enumeration.
\end{corollary}

We now state the main result of this section.

\begin{theorem}[\mbox{$\CalC\in[\It\Inf\Ex]$}]
	\label{ItInfHalfspace}
	The concept class of halfspaces $\CalC$ is learnable by an iterative learner.
\end{theorem}

For the rest of this section we sketch the argument for $d=2$ and refer the interested reader to Section~\ref{sec:halfSpacesgen} for a general proof.

\medskip
With the help of the following definition, we can give another uniform decision procedure for $\CalH$, to which the iterative learner will refer.
This procedure allows the iterative learner to store a finite amount of information as part of its current hypothesis.

\begin{definition}[$\LOCK$ property for $\u,\ve,\x,\y\in\Z\times\Z$]
	Let $\u,\ve,\x,\y$ lie on the two-dimensional integer grid, $\Z\times\Z$.
	The four points \emph{$\u,\ve,\x,\y$ have the $\LOCK$-property} if
	\begin{enumerate}
		\item $\u\neq\ve$ and $\x\neq\y$,
		\item the lines through $\u, \ve$ and $\x,\y$ are parallel, in particular distinct,
		\item the lines through $\u,\ve$ and $\x,\y$ are of minimal distance with respect to the integer grid,
			i.e. there is no parallel line passing through an integral point and strictly between them,
		\item\label{SignSlope} there is a point on the line segment between $\u,\ve$, 
			such that the corresponding points with the same first/second coordinate on the line through $\x,\y$ lie on the line segment between $\x,\y$.
	\end{enumerate}
\end{definition}

Note that \ref{SignSlope}. implies that
	\begin{enumerate}
		\item[5.]\label{DecrDist} the minimal distance is realized between the line segments $\overline{\u\ve}$ and $\overline{\x\y}$.
	\end{enumerate}

\begin{lemma}\label{Dist}
	Let $a_1,a_2\in\Z$ such that $\gcd(a_1,a_2)=1$.
	Then the minimal distance between distinct lines with normal vector $(a_1,a_2)$ passing through integral points is $\frac{1}{\sqrt{a_1^2+a_2^2}}$.
\end{lemma}
\begin{proof}
We denote by $|a|$ the distance between $a$ and $0$, e.g. $|2|=|-2|=2$.
The minimal horizontal/vertical distances between two lines with normal vector $(a_1,a_2)$ passing through integral points are $\frac{1}{|a_1|}$ and $\frac{1}{|a_2|}$, respectively.
From this follows that the minimal distance between the lines is as claimed.
\end{proof}

As we encode integers and vectors (of vectors) of integers into natural numbers, we transfer the definition of the $\LOCK$ property to natural numbers.

\begin{definition}[$\LOCK$ Property for $j\in\N$] \label{jLock}
	Let $j\in\N$. Extract four points $\u,\ve,\x,\y$ on the two-dimensional integer grid from $j$.
	(As $\u=(u_1,u_2), \ldots, \y=(y_1,y_2)\in\Z\times\Z$, this can be done with a repeated application of the computable inverse by assuming
	$j=\langle\langle\langle u_1 \rangle,\langle u_2 \rangle\rangle,\langle\langle v_1 \rangle,\langle v_2 \rangle\rangle,
	\langle\langle x_1 \rangle,\langle x_2 \rangle\rangle,\langle\langle y_1 \rangle,\langle y_2 \rangle\rangle\rangle$.)
	We say that \emph{$j$ has the $\LOCK$ property}, if $\u,\ve,\x,\y$ have the $\LOCK$ property.
\end{definition}

We now describe the uniform decision procedure to which the iterative learner will refer.

Basically, the first coordinate of the input tells whether the learner thinks it is finished or is in data collection mode.
If it thinks it is finished, it interprets the coordinate as 4 points on the integer grid.
If these four points are candidates for defining the prediction model to be learned, then the decision procedure computes a halfspace from them.
It then checks whether the point given by the second coordinate of the input fits the halfspace.
If the four points are no valid candidates or the learner is in data collection mode, the decision procedure will treat it as a hypothesis for the upper halfplane (second coordinate $\geq 0$), which simply serves as a dummy hypothesis.

More formally, assume the input of the decision procedure are natural numbers $i,n\in\N$.
If $i=2j+1$ for $j\in\N$, this is interpreted as maybe being finished.
Then the procedure checks whether $j$ has the $\LOCK$ property.
If it does, the decision procedure computes $a_0,a_1,a_2$ for the halfspace given by $\ell_{\u,\ve}$, while assuming that $\x,\y$ are not in the halfspace. (For the definition of $\u,\ve,\x,\y$, see Definition~\ref{jLock}.)
Next, it extracts $\z=(z_1,z_2)\in\Z\times\Z$ such that for the second input $n$ holds $n=\langle\langle  z_1 \rangle,\langle z_2 \rangle\rangle$.
Finally, the procedure checks whether $a_0 \geq a_1z_1+a_2z_2$ and returns $1$ if the inequality is true.
In all other cases the decision procedure returns $1$ if $z_2\geq 0$.

Note that for every odd number $2j+1$, with $j\in\N$ having the property $\LOCK$, the prediction model $f_{2j+1}$ represents the unique halfspace $L_{a_0,a_1,a_2}$ with normal vector $(a_1,a_2)$, $\gcd(a_1,a_2)=1$, and displacement $a_0$ corresponding to $\ell_{\u,\ve}$ and $(a_1,a_2)$ pointing towards $\x,\y$.

Moreover, all prediction models $f_i$ for $i$ even or $i=2j+1$ with $j$ \emph{not} having property $\LOCK$
refer to $L_{0,0,-1}=\{\langle\langle z_1\rangle,\langle z_2\rangle\rangle \mid z_2\geq 0 \}$.

\bigskip
Now, we define the iterative learner $M$ for $\CalC$.
Initialize with $0$.

If the learner is in data collection mode, check whether the stored data together with the new datum contains points $\u,\ve$ positively labeled and $\x,\y$ negatively labeled with $\u,\ve,\x,\y$ having property $\LOCK$.
If not, simply add the new datum to the stored data and stay in data collection mode.
If yes, switch to the maybe finished mode and store witnessing $\u,\ve,\x,\y$.

If the learner is in maybe finished mode, i.e. its last hypothesis is $2j+1$, check whether the new datum is consistent with the halfspace corresponding to $L_{2j+1}$.
If not, the learner switches to the data collection mode and stores
$$\langle\langle\u\rangle,1\rangle,\langle\langle\ve\rangle,1\rangle,\langle\langle\x\rangle,0\rangle,\langle\langle\y\rangle,0\rangle$$
and the new datum $\sigma(|\sigma|-1)$.
If yes, the learner repeats its last hypothesis $2j+1$ and therefore forgets the current datum.

Formally, $M$ is initialized with the hypothesis $0$ standing for $L_{0,0,-1}$.
Let $\sigma\in \mathbb{S}$, $|\sigma|>0$.
Then $\sigma^-$ denotes $\sigma$ without its last element $\sigma(|\sigma|-1)=(\langle\w\rangle,\lambda)$.

If $M(\sigma^-)=2j$ is even, the learner extracts from $j$ two numbers $s$, $w$.
With the interpretation of $s$ to be $w$'s length, it extracts from $w$ the stored data
$$(\langle\w_1\rangle,\lambda_1),\ldots,(\langle\w_s\rangle,\lambda_s)\in\N\times\{0,1\}.$$
The learner now considers the set
$W=\{ \w,\w_1,\ldots,\w_s\}$.
Now, if there are $\u,\ve\in W$ positively labeled and $\x,\y\in W$ negatively labeled with the property $\LOCK$,
the learner outputs the hypothesis
$$2\langle\langle \u\rangle,\langle \ve\rangle,\langle \x\rangle,\langle \y\rangle\rangle+1.$$ 
If there are no such witnesses for the property $\LOCK$, especially if $s<3$, it outputs
$$2\langle s+1, \langle\langle\langle\w_1\rangle,\lambda_1\rangle,\ldots,\langle\langle\w_s\rangle,\lambda_s\rangle,\langle\langle\w\rangle,\lambda\rangle\rangle,$$
i.e., appends the new datum to the array of stored labeled data.

If $M(\sigma^-)=2j+1$ is odd, the learner extracts $\u,\ve,\x,\y$ from $j$ and checks whether the new datum $(\langle\w\rangle,\lambda)$ is consistent with the halfspace corresponding to the four points.
If not, the learner switches to data collection mode by outputting
$$2\langle 5, \langle\langle\langle\u\rangle,1\rangle,\langle\langle\ve\rangle,1\rangle,\langle\langle\x\rangle,0\rangle,\langle\langle\y\rangle,0\rangle,\langle\langle\w\rangle,\lambda\rangle\rangle.$$
Otherwise, it repeats its last hypothesis
$$2j+1.$$

\smallskip
The learner converges for the following reasons:

If the learner is first locked on a halfspace with positive/negative slope, then all other slopes corresponding to locking hypotheses will be positive/negative, due to (\ref{SignSlope}).
This holds due to the size of the overlap of the defining positive/negative line segments of a locking hypothesis.
In more detail, because $a_1$ and $a_2$ are greater or equal 1, $\frac{1}{a_1}$ is less or equal to $a_2$.

If the halfspace $L$ to be learned is vertical or horizontal, the learner will never reach a locking hypothesis $2j+1$ with $f_{2j+1}$ not corresponding to $L$.

Due to (\ref{DecrDist}) the sequence of locking distances is strictly decreasing and bounded from below by the minimal distance corresponding to the halfspace $L$ to be learned.
Hence the learner will never lock on a hypothesis with the same corresponding normal vector $(a_1,a_2)$ with $\gcd(a_1,a_2)=1$ as a previously discarded locking hypothesis again and there are only finitely many choices for $(a_1,a_2)$ due to the lower bound on the value of the distance function given by Lemma~\ref{Dist}.

The learner will finally learn $L$ because for every locking hypothesis $2j+1$ not corresponding to $L$, there are infinitely many positively and infinitely many negatively labeled points in $\Z\times\Z$, labeled with respect to $L$, and not consistent with $L_{2j+1}$. Hence, having discarded finitely many is not be problematic.

For every halfspace $L$ and every informant for $L$, the observations immediately yield the success of the iterative learning algorithm.

\section{Proof for the Learnability of Half-Spaces in Arbitrary Dimension}
\label{sec:halfSpacesgen}

We now formaly define the concepts involved for arbitrary dimension $d>0$.

\begin{definition}
A hyperplane $H$ in a d-dimensional space is described by an equation \begin{equation}\label{hyperplane} \sum_{i=1}^d a_i \cdot x_i + a_0 = 0 \end{equation} that is satisfied by all its points $p = (x_1, \dots, x_d)$. In this equation $a_1, \dots, a_d$ are called the slope coefficients and $a_0$ is the displacement.
\end{definition}

\begin{lemma}\label{integ}
Let $H$ be a hyperplane in a $d$ dimensional space with rational slope coefficients, that is, any point $p = (x_1, \dots, x_d)$ on $H$ satisfies $\sum_{i=1}^d r_i \cdot x_i + r_0 = 0$ where the $r_i$ are rational numbers. The points on $H$ then also satisfy an equation $\sum_{i=1}^d a_i \cdot x_i + a_0 = 0$ where the coefficients $a_1, \dots, a_d$ are integers such that $\gcd(a_i, \dots, a_d) = 1$. $a_0$ is also an integer if and only if $H$ passes through an integral point.
\begin{proof}
This is achieved by multiplying the equation $\sum_{i=1}^d r_i \cdot x_i + r_0 = 0$ by $\lcm(q_1, \dots, q_d)$ and dividing it by $\gcd(p_1, \dots, p_d)$ where $r_i = p_i/q_i$ is a reduced fraction meaning $\gcd(p_i, q_i) = 1$. Since $q_i \divides \lcm(q_1, \dots, q_d)$ the $a_i$'s turn out integers.
To see that $\gcd(a_i, \dots, a_d) = 1$ assume there is an integer $c$ that divides $\frac{p_i}{\gcd(p_1, \dots, p_d)} \cdot \frac{\lcm(q_1, \dots, q_d)}{q_i}$ for all $i$.
Because of prime decomposition, we might assume that $c$ is prime. By definition of greatest common divisor, it can not be that $c \divides \frac{p_i}{\gcd(p_1, \dots, p_d)}$ for all $i$.
This means there exists a $j$ such that $c \notdivides \frac{p_j}{\gcd(p_1, \dots, p_d)}$ so by primality of $c$ we must have $c \divides  \frac{\lcm(q_1, \dots, q_d)}{q_j}$.
This in turn means by the definition of least common multiple that there exists a $k$ such that $c \divides q_k$. Now let $q_l$ be divisible by the highest power of $c$.
This means $c \notdivides \frac{\lcm(q_1, \dots, q_d)}{q_l}$ and of course that $c \divides q_l$. Since fractions were reduced we have $\gcd(p_l, q_l) = 1$ meaning $c \notdivides p_l$.
This implies $c \notdivides \frac{p_l}{\gcd(p_1, \dots, p_d)}$ and therefore $c \notdivides \frac{p_l}{\gcd(p_1, \dots, p_d)} \cdot \frac{\lcm(q_1, \dots, q_d)}{q_l}$ contrary to assumption.

For the last statement, note that if there are integer $x_i$ satisfying the equation, by integrality of $a_1, \dots, a_d$ we get that $a_0$ must be integer. For the converse, suppose that $a_0$ is an integer. Since $\gcd(a_1, \dots, a_d) = 1$ there are by Bezout's identity integral coefficients $y_1, \dots, y_d$ such that $\sum_{i=1}^d y_i \cdot a_i = 1$. Setting $x_i = a_0 \cdot y_i$ we have the desired coordinates of an integral point on the hyperplane $H$.
\end{proof}
\end{lemma}

\begin{definition}
A hyperplane with defining equation $\sum_{i=1}^d a_i \cdot x_i + a_0 = 0$ where the coefficients $a_1, \dots, a_d$ are integers such that $\gcd(a_i, \dots, a_d) = 1$ is said to be in integral reduced form.
\end{definition}

\begin{definition}
The $j-$distance of a point $p$ to a hyperplane $H$ is the distance of $p$ to a point $q$ on the plane $H$ that has all coordinates but the $j$th equal to those of $p$. If such a $q$ does not exist the $j-$distance is undefined (or $\frac{1}{0}$).
\end{definition}

\begin{lemma}\label{jdist}
Let $H$ be a hyperplane with slope coefficients $a_i$ in integral reduced form which passes through an integral point. The smallest $j-$distance to $H$ of an integral point not on $H$ is equal to $1/a_j$. Furthermore, such ``$j-$closest'' points to $H$ not on the hyperplane can be found on both sides of $H$.
\begin{proof}
Rewriting the defining equation for $H$ we get for the $j-$th coordinate \begin{equation}\label{x_j} x_j = -\frac{1}{a_j} \left[\sum_{i=1, \neq j}^d a_i \cdot x_i + a_0\right].\end{equation} Define $b = \gcd(\{a_1, \dots, a_d\} \setminus \{a_j\})$. This means that by Bezout's identity there are integers $y_i$ such that $\sum_{i=1, \neq j}^d a_i \cdot y_i = m \cdot b$ for any integer multiple $m$ of $b$. Since $\gcd(a_1, \dots, a_d) = 1$ we must have $\gcd(b, a_j) = 1$, meaning there is an integer $m$ such that $m \cdot b \overset{\mod a_j}{=} 1$ or equivalently, there exist integers $m$ and $n$ such that $m \cdot b = n \cdot a_j + 1$. So if the $y_i$ were the values s.t. $\sum_{i=1, \neq j}^d a_i \cdot y_i = m \cdot b$, we have by setting the integer valued coordinates $x_i = (\pm1 - a_0) \cdot y_i$ that $x_j = -\frac{1}{a_j} \left[(\pm1 - a_0) \cdot n \cdot a_j \pm 1 - a_0+ a_0\right] = (a_0 \mp 1) \cdot n \mp \frac{1}{a_j}$. The integral points having $i$th coordinates $x_i$ (in each case) for $i \neq j$ and $j$th coordinate equal to $(\pm a_0 - 1) \cdot n$ have $j-$distance $\frac{1}{a_j}$ to plane $H$ on the two different sides of it. One can easily see that a smaller $j-$distance is not possible for integral points due to equation \ref{x_j} for the $j$th coordinate of points on $H$.
\end{proof}
\end{lemma}

\begin{lemma}\label{ortho}
Assume we have pairwise orthogonal vectors $v_i$ for $i = 1, \dots, d$ in a $d-$dimensional space, and let $H$ be the hyperplane passing through the heads of these vectors when their tails are placed on the origin. Then the vector $h$ from the origin to $H$ and orthogonal to it is equal to $\frac{\sum_{i=1}^d v_i/v_i^2}{\sum_{i=1}^d 1/v_i^2}$.
\begin{proof}
By definition we must have $(v_i - h) \cdot h = 0$ for all $i$. This implies $|h|^2 = h \cdot v_i$ for all $i$. If we expand $h$ in the basis of the $v_i$ we have $h = (h_1, \dots, h_d)$ and so $h_i|v_i| = |h|^2$ for all $i = 1, \dots, d$. This means $h = |h|^2 \cdot (\frac{1}{|v_i|}, \dots, \frac{1}{|v_d|})$. Taking the inner product with itself we get $|h|^2 = |h|^4 \cdot \sum_{i=1}^d 1/v_i^2 \Rightarrow |h|^2 = 1/\sum_{i=1}^d 1/v_i^2$ which proves the statement.
\end{proof}
\end{lemma}

\begin{corollary}\label{norm}
The vector $h$ as in lemma \ref{ortho} has norm $\frac{1}{\sqrt{\sum_{i=1}^d 1/v_i^2}}$.
\begin{proof}
Follows from lemma \ref{ortho}.
\end{proof}
\end{corollary}

\begin{theorem}\label{mindist}
Let $H$ be a hyperplane with integral slope coefficients $a_i$ in integral reduced form which passes through an integral point. The closest parallel hyperplanes to it passing through different integral points have a distance of $1/\sqrt{\sum_{k=1}^d a_k^2}$ to it.
\begin{proof}
By lemma \ref{jdist} the distance along the $j$th axis to these hyperplanes is equal to $1/a_j$. By corollary \ref{norm} the orthogonal distance between two closest such parallel hyperplanes will be
\[ \left( \sum_{k=1}^d  \frac{1}{1/a_k^2} \right)^{-1/2} = \left( \sum_{k=1}^d a_k^2 \right)^{-1/2} \]
\end{proof}
\end{theorem}

\begin{definition}
The integral half grid problem consists of a ground set $G_d = \mathbb{Z}^d$, the integral grid in $d$ dimensions, and a class of half-spaces $\mathcal{L}_{Ihg}$ which consists of a half-space for every hyperplane with rational slope coefficients. For every $(r_1, \dots, r_d, \Delta_0)$ where $r_1, \dots, r_d \in \mathbb{Q}$ and $\Delta_0 \in \mathbb{R}$ the language $L_{(r_1, \dots, r_d, \Delta_0)} \in \mathcal{L}_{Ihg}$ consists of all points $p = (x_1, \dots, x_d) \in \mathbb{Z}^d$ such that $\sum_{i=1}^d r_i \cdot x_i + \Delta_0 \ge 0$. The problem is now for a learner to identify a target $L_t \in \mathcal{L}_{Ihg}$ in the limit.
\end{definition}

\begin{lemma}\label{equiv}
In the integral half grid problem there is a one to one correspondence between languages in $\mathcal{L}_{Ihg}$ and the elements of $\mathbb{Z}^{d+1}$. Specifically, after putting the defining equations of hyperplanes corresponding to all languages $L \in \mathcal{L}_{Ihg}$ in integral reduced form, the one to one correspondence will be between distinct languages (half-spaces) of $\mathcal{L}_{Ihg}$ and equivalence classes of the coefficients defined by taking the integer part of the displacements $(a_1, \dots, a_d, \lfloor a_0 \rfloor)$. In particular, if two languages $L, L' \in \mathcal{L}_{Ihg}$ have coefficients in integral reduced form $a$ and $a'$ such that $a_i = a'_i$ for $1 \le i \le d$ and $\lfloor a_0 \rfloor = \lfloor a'_0 \rfloor$ then these two languages are identical $L = L'$.
\begin{proof}
For any integral point $p = (x_1, \dots, x_d)$ satisfying $\sum_{i=1}^d a_i \cdot x_i + a_0 \ge 0$ we may take integer parts from both sides to obtain $\sum_{i=1}^d a_i \cdot x_i + \lfloor a_0 \rfloor \ge 0$. Conversely, it is clear that since $\lfloor a_0 \rfloor \le a_0$, that $\sum_{i=1}^d a_i \cdot x_i + \lfloor a_0 \rfloor \ge 0$ implies $\sum_{i=1}^d a_i \cdot x_i + a_0 \ge 0$.
\end{proof}
\end{lemma}

\begin{definition}
A basic set in d-dimensional space is a set of $d$ affine-independent integral points, i.e. $C = \{c_0, \dots, c_{d-1}\}$ s.t. the vectors $c_i - c_0$ for $i = 1, \dots, d-1$ are linearly independent. The unique ($d-1$-dimensional) hyperplane $H_c$ passing through the points of $C$ is simply called $C$'s hyperplane and $C$ is a basic set for $H_C$. A basic cell $\conv (C)$ is the convex hull of points in a basic set $C$. Two basic sets $C$ and $C'$ are parallel if their hyperplanes are, they are facing each other if they are parallel and there is a line segment orthogonal to their hyperplanes meeting their cells, that is, there are points $p \in \conv(C)$ and $p' \in \conv(C')$ such that $[p, p']$ is orthgonal to $H_C$ and $H_{C'}$. Two basic sets are adjacent if they are facing each other and their hyperplanes are distinct but as close as possible, having the distance from theorem \ref{mindist}.
\end{definition}

\begin{lemma}\label{pair}
Suppose a language (half-space) $L \in \mathcal{L}_{Ihg}$ is determined by a hyperplane $H$ with coefficients $a$ in integral reduced form such that all grid points $p = (x_1, \dots, x_d) \in L$ satisfy $\sum_{i=1}^d a_i \cdot x_i + a_0 \ge 0$. We then have in addition to all grid points in $L$ satisfying $\sum_{i=1}^d a_i \cdot x_i + \lfloor a_0 \rfloor \ge 0$ as stated in lemma \ref{equiv}, that all grid points not contained in this halfspace $q = (y_1, \dots, y_d) \in L^c$ satisfy $\sum_{i=1}^d a_i \cdot y_i + \lfloor a_0 \rfloor + 1 \le 0$ or equivalently, \[\sum_{i=1}^d (-a_i) \cdot y_i + (- \lfloor a_0 \rfloor - 1) \ge 0.\] Furthermore, both these inequalities are tight in the sense that they are satisfied with equality for elements of $L$ and $L^c$ respectively.
\begin{proof}
According to lemma \ref{equiv} we must have for every $q = (y_1, \dots, y_d) \in L^c$ that $\sum_{i=1}^d a_i \cdot y_i + \lfloor a_0 \rfloor < 0$. Since the coordinates of $q$ are integral we have $\sum_{i=1}^d a_i \cdot y_i \in \mathbb{Z}$ and because $\mathbb{Z} \cap \mathbb{R}_{<0} = \mathbb{Z}_{\le -1}$ we must have $\sum_{i=1}^d a_i \cdot y_i + \lfloor a_0 \rfloor \le -1$ proving the statement.

For the second statement, notice that the $a$ coefficients are in integral reduced form meaning $\gcd(a_1, \dots, a_d) = 1$ so that by Bezout's identity there are integral coordinates $(z_1, \dots, z_d)$ such that $\sum_{i=1}^d a_i \cdot z_i = k$ for any integer $k \in \mathbb{Z}$.
\end{proof}
\end{lemma}

\begin{definition}\label{tang}
For a hyperplane $H$ described by an integral reduced form $\sum_{i=1}^d a_i \cdot x_i + a_0 = 0$ we define its positive tangent $H_+$ as the halfspace described by the inequality $\sum_{i=1}^d a_i \cdot x_i + \lfloor a_0 \rfloor \ge 0$ and and its negative tangent $H_-$ as the halfspace described by the inequality $\sum_{i=1}^d (-a_i) \cdot x_i + (- \lfloor a_0 \rfloor - 1) \ge 0$.
\end{definition}

\begin{corollary}
If a hyperplane $H$ separates points in $L$ from points in $L^c$ of the integral grid which we could see as positive and negative points, the hyperplanes tangent to the positive and negative points are exactly the boundaries of $H_+$ and $H_-$ as in definition \ref{tang}.
\begin{proof}
Follows from lemma \ref{pair}.
\end{proof}
\end{corollary}

\begin{definition}
We will be considering a hypothesis space consisting of sets of positive and negative data points $\mathcal{H} = \{ \{ (p, s) | p \in \mathbb{Z}^d, s \in \{+, -\} \} \}$. A locked state is achieved when for a hypothesis $H = \{(p, s) : p \in \mathbb{Z}^d, s \in \{+, -\}\}$ a subset $C_+$ of the positive points of $H$ and a subset $C_-$ of the negative points of $H$ form adjacent basic sets such that all other data points retained in the hypothesis are separated based on sign by the hyperplanes of these two cells $H_{C_+}$ and $H_{C_-}$ meaning $H_{C_+}$ is the boundary of a half-space $H^{C_+}_+$ and $H_{C_-}$ is the boundary of a half-space $H^{C_-}_-$ such that $H_+ \cap H_- = \emptyset$ and for all $(p, +) \in H$ we have $p \in H_+$ and for all $(q,-) \in H$ we have $q \in H_-$. The distance of a locked state $d$ is the distance between $H_{C_+}$ and $H_{C_-}$.
\end{definition}

\begin{definition}
The violation of a locked states happens by receiving a data point $(p, s)$ that does not respect separation by the hyperplanes of the adjacent basic sets, meaning it is on the other side of these hyperplanes than data points of the same sign as it, either $p \in H^{C_-}_-$ for data point $(p, +)$ or $q \in H^{C_+}_+$ for data point $(q, -)$. Remember that there are no integral points strictly between hyperplanes of adjacent basic sets by definition of their respective hyperplanes being as close as possible.
\end{definition}

\begin{algorithm2e}[H]
	Initialize $H \assign \emptyset$, $\state$ $\assign$ $\open$\;
	Receive new data point $(p, s) : p \in \mathbb{Z}^d, s \in \{+, -\}$\;
	\uIf{$\state = \open$ }{
		$H \assign H \cup (p,s)$\;
		\If{$H$ is a locked state}{
			$\state \assign \locked$\;
			\label{conv}Apply convention: either do nothing or discard all previous data not required for this locked state\;
		}
	} \ElseIf{$\state = \locked$}{
		\If{$p$ violates the locked state}{
			$\state \assign \open$\;
			$H \assign H \cup (p,s)$\;
		}
	}
	\caption{Iterative learner of integral half-spaces from informants}
	\label{alg:hsit}
\end{algorithm2e}

\begin{lemma}\label{decr}
If $d$ is the distance of a locked state at some point in algorithm \ref{alg:hsit} which is afterwards violated by a data point and $d'$ is the distance of a later locked state we have $d > d'$. That is, the distance of locked states is strictly decreasing.
\begin{proof}
Assume $H_+$ and $H_-$ are the half-spaces of the first locked state of distance $d$ and $H'_+$ and $H'_-$ are the half-spaces of the second locked state of distance $d'$. The sign indices indicate in both cases the signs of the data points of the corresponding basic cells. Since all data points respect the separation by the two hyperplanes in the new locked state including the points of the basic cells of the first locked state, we have the distance of any positive point and any negative point in the first locked state is at least $d'$. This gives us that $d \ge d'$ because by definition of adjacency the previous locked state had basic sets facing each other, meaning there were points $p$ and $q$ in the associated basic cells of distance $d$ where $p$ was a convex combination of positive points and $q$ a convex combination of negative points. Since all positive points are now in $H'_+$ and all negative points are in $H'_-$ the same holds for convex combinations of each label of points and thus $d \ge d'$. If we were to have equality $d = d'$ that would mean that the facing points $p$ and $q$ from the basic cells of the first locked state are situated exactly on the boundaries of $H'_+$ and $H'_-$, and because $[p,q]$ is orthogonal to the boundaries of $H_+$ and $H_-$, we must have $H_+ = H'_+$ and $H_- = H'_-$ which would contradict the first locked state ever being violated in the first place thereby proving $d \gneq d'$.
\end{proof}
\end{lemma}

\begin{definition}
The target distance $d_t$ is the orthogonal distance between the tangents $H^t_+$ and $H^t_-$ for the hyperplane $H^t$ associated with the target language (half-space) $L_t$.
\end{definition}

\begin{lemma}\label{bnd}
The distance of any locked state is bounded from below by the target distance.
\begin{proof}
Similar to the proof of lemma \ref{decr} since all data points respect separation by $H^t_+$ and $H^t_-$.
\end{proof}
\end{lemma}

\begin{lemma}\label{lock}
If the learner of algorithm \ref{alg:hsit} is in state $\open$ it will eventually go into $\locked$.
\begin{proof}
In the $\open$ state all incoming data points are received and aggregated and none is refused. By whatever convention for the $\locked$ state in which we may have discarded previous data points, we have two cases:
\begin{enumerate}
\item The learner eventually goes into a locked state with tangents different from that of the target's
\item Not case 1
\end{enumerate}
In the second case, assume all previously received data points (which there are finitely many of) are contained in a bounded ball $B$. Even if all preveious points were discareded based on convention in line \ref{conv} of algorithm \ref{alg:hsit}, there will still be infinitely many data points on $H^t_+$ and $H^t_-$ further away from $B$ which will be received and eventually create adjacent basic cells which force the learner into the $\locked$ state with the true target tanget hyperplanes.
\end{proof}
\end{lemma}

\begin{lemma}\label{inf}
If two languages (half-spaces) $L, L' \in \mathcal{L}_{Ihg}$ are distinct, there will be grid points in their symmetric difference $L \Delta L'$ arbitrarily distant from any compact set $B$.
\begin{proof}
For this we make a case distinction:
\begin{enumerate}
\item $L$ and $L'$ have identical slope coefficients
\item $L$ and $L'$ don't have identical slope coefficients
\end{enumerate}
In the first case, distinction of the two half-spaces can only mean their displacements in integral reduced form having different integer parts. We know there exists at least one point $p_0$ labeled differently by the two languages. There are infinitely many integral translation vectors $\delta = (\delta_1, \dots, \delta_d) \in \mathbb{Z}^d$ that satisfy $\sum_{i=1}^d a_i \cdot \delta_i = 0$ and for each one of them $p_0 + \delta$ would also be labeled differently by $L$ and $L'$.

In the second case, consider the two vectors $a = (a_1, \dots, a_d)$ and $a' = (a'_1, \dots, a'_d)$ of the coefficients of the two half-spaces in integral reduced form.
They are in integral reduced form but different which implies $a \nparallel a'$.
This enables us to find an integral vector $b$ such that $b.a$ and $b.a'$ are both nonzero and of opposite signs.
W.l.o.g. assume we have a point $p_0$ classified by $L$ as positive and by $L'$ as negative and that $b.a > 0$ while $b.a' < 0$ (otherwise take $-b$).
Now all points $p_0 + m \cdot b$ for $m \in \mathbb{N}$ will be classified as positive by $L$ and negative by $L'$.
\end{proof}
\end{lemma}

\begin{lemma}\label{unlock}
If the learner from algorithm \ref{alg:hsit} goes into a $\locked$ state with tangent hyperplanes other than that of the target's, the $\locked$ state will eventually be violated.
\begin{proof}
If all previously received data points (which there are finitely many of) are contained in a bounded ball $B$, there will still be infinitely many data points further away from $B$ corresponding to the true target $H^t$. But by lemma \ref{inf} any two distinct hyperplanes will label some points differently arbitrarily distant from any compact set $B$. Therefore, a new data point labeled inconsistently with the separation of the current $\locked$ state will eventually be received by the learner, violating the $\locked$ state and causing the learner to transition to state $\open$.
\end{proof}
\end{lemma}

\begin{lemma}\label{fin}
The learner from algorithm \ref{alg:hsit} goes into finitely many $\locked$ states in total.
\begin{proof}
By lemma \ref{decr} the distance of locked states strictly decrease and by lemma \ref{bnd} they are bounded from below. By lemma \ref{mindist} these distances can only assume certain discrete values and the total set of combinations of the slope coefficients providing distances at least that of the target distance $d^t$ is finite because they need to satisfy $ \sum_{i=1}^d a_i^2 \le 1/{d^t}^2$. 
\end{proof}
\end{lemma}

\begin{theorem}
The learner from algorithm \ref{alg:hsit} identifies the target (tanget) hyperplane in a finite number of steps.
\begin{proof}
By lemma \ref{lock} it will never remain in an $\open$ state indefinitely, and by lemma \ref{unlock} it will eventually come out of any $\locked$ state which does not correspond to the target. But by lemma \ref{fin} the learner goes into state $\locked$ only finitely many times, so it must eventually go into a $\locked$ state that does correspond to the target. By algorithm \ref{alg:hsit} the hypothesis remains constant as long as the learner remains in $\locked$ state, so if the $\locked$ state refers to the half grid it corresponds to, algorithm \ref{alg:hsit} is able to learn the class of integral half grids in the limit.
\end{proof}
\end{theorem}

\subsection*{Acknowledgements}
{We are grateful to the people supporting us.
Especially, the third author thanks André Nies for pointing out the idea to study linear functions and Eugen Hellmann, Sanjay Jain, Peter Scholze, Frank Stephan and Simon Wietheger for helpful feedback regarding early forms or isolated parts of the proof for the learnability of halfspaces by this constructive iterative learner.
Moreover, the first and the last author thank Vanja Dosko\v{c} and Armin Wells for helpful discussions of proof ideas for the learnability.
We thank Thomas Zeugmann and Sandra Zilles for pointers to prior research. \\
This work was supported by DFG Grant Number KO 4635/1-1. }

\bibliographystyle{alpha}

\bibliography{CLTBib}

\newcommand{\etalchar}[1]{$^{#1}$}
\begin{thebibliography}{BCM{\etalchar{+}}08}

\bibitem[AKS18]{As-Koe-Sei2018_informants}
M.~Aschenbach, T.~K{\"o}tzing, and K.~Seidel.
\newblock Learning from informants: Relations between learning success
  criteria.
\newblock {\em arXiv preprint arXiv:1801.10502}, 2018.

\bibitem[Ang80]{angluin1980inductive}
D.~Angluin.
\newblock Inductive inference of formal languages from positive data.
\newblock {\em Information and control}, 45(2):117--135, 1980.

\bibitem[B{\=a}r77]{Bar:c:77:aut-func-prog}
J.~B{\=a}rzdi\c{n}\v{s}.
\newblock Inductive inference of automata, functions and programs.
\newblock In {\em Amer. Math. Soc. Transl.}, pages 107--122, 1977.

\bibitem[BB75]{Blu-Blu:j:75}
L.~Blum and M.~Blum.
\newblock Toward a mathematical theory of inductive inference.
\newblock {\em Information and Control}, 28:125--155, 1975.

\bibitem[BCM{\etalchar{+}}08]{Bal-Cas-Mer-Ste-Wie:j:08}
G.~Baliga, J.~Case, W.~Merkle, F.~Stephan, and R.~Wiehagen.
\newblock When unlearning helps.
\newblock {\em Information and Computation}, 206:694--709, 2008.

\bibitem[CJLZ99]{Cas-Jai-Lan-Zeu:j:99:feedback}
J.~Case, S.~Jain, S.~Lange, and T.~Zeugmann.
\newblock Incremental concept learning for bounded data mining.
\newblock {\em Information and Computation}, 152:74--110, 1999.

\bibitem[CK10]{Cas-Koe:c:10:colt}
J.~Case and T.~K{\"o}tzing.
\newblock Strongly non-{U}-shaped learning results by general techniques.
\newblock In Adam~Tauman Kalai and Mehryar Mohri, editors, {\em {COLT} 2010 -
  The 23rd Conference on Learning Theory, Haifa, Israel, June 27-29, 2010},
  pages 181--193. Omnipress, 2010.

\bibitem[CM07]{Cas-Moe:c:07:nonuit}
J.~Case and S.~Moelius.
\newblock {U}-shaped, iterative, and iterative-with-counter learning.
\newblock In N.~Bshouty and C.~Gentile, editors, {\em Proceedings of the 20th
  Annual Conference on Learning Theory (COLT'07)}, volume 4539 of {\em Lecture
  {N}otes in {A}rtificial {I}ntelligence}, pages 172--186, 2007.

\bibitem[CM08]{Cas-Moe:j:08:nonuit}
J.~Case and S.~E. Moelius.
\newblock U-shaped, iterative, and iterative-with-counter learning.
\newblock {\em Machine Learning}, 72:63--88, 2008.

\bibitem[CM09]{CaseMoelius2009}
J.~Case and S.~Moelius.
\newblock Parallelism increases iterative learning power.
\newblock {\em Theoretical Computer Science}, 410(19):1863 -- 1875, 2009.

\bibitem[CM11]{Cas-Moe:j:11:optLan}
J.~Case and S.~Moelius.
\newblock Optimal language learning from positive data.
\newblock {\em Information and Computation}, 209:1293--1311, 2011.

\bibitem[Gol67]{Gol:j:67}
E.~Gold.
\newblock Language identification in the limit.
\newblock {\em Information and Control}, 10:447--474, 1967.

\bibitem[GRSZ17]{gao2017preference}
Z.~Gao, C.~Ries, H.~U. Simon, and S.~Zilles.
\newblock Preference-based teaching.
\newblock {\em The Journal of Machine Learning Research}, 18(1):1012--1043,
  2017.

\bibitem[Jan91]{j-mniifp-91}
K.~P. Jantke.
\newblock Monotonic and nonmonotonic inductive inference of functions and
  patterns.
\newblock In {\em Nonmonotonic and Inductive Logic, 1st International Workshop,
  Proc.}, pages 161--177, 1991.

\bibitem[JKMS16]{jain2016role}
S.~Jain, T.~K{\"o}tzing, J.~Ma, and F.~Stephan.
\newblock On the role of update constraints and text-types in iterative
  learning.
\newblock {\em Information and Computation}, 247:152--168, 2016.

\bibitem[JLZ07]{Jai-Lan-Zil:j:07}
S.~Jain, S.~Lange, and S.~Zilles.
\newblock Some natural conditions on incremental learning.
\newblock {\em Information and Computation}, 205:1671--1684, 2007.

\bibitem[JMZ13]{Jai-Moe-Zil:j:13}
S.~Jain, S.~Moelius, and S.~Zilles.
\newblock Learning without coding.
\newblock {\em Theoretical Computer Science}, 473:124--148, 2013.

\bibitem[JORS99]{Jai-Osh-Roy-Sha:b:99:stl2}
S.~Jain, D.~Osherson, J.~Royer, and A.~Sharma.
\newblock {\em Systems that Learn: {A}n Introduction to Learning Theory}.
\newblock MIT Press, Cambridge, Massachusetts, second edition, 1999.

\bibitem[K{\"o}t09]{Koe:th:09}
T.~K{\"o}tzing.
\newblock {\em Abstraction and Complexity in Computational Learning in the
  Limit}.
\newblock PhD thesis, University of Delaware, 2009.

\bibitem[KP14]{kotzing2014map}
T.~K{\"o}tzing and R.~Palenta.
\newblock A map of update constraints in inductive inference.
\newblock In {\em Algorithmic Learning Theory}, pages 40--54, 2014.

\bibitem[KS16]{kotzing2016towards}
T.~K{\"o}tzing and M.~Schirneck.
\newblock Towards an atlas of computational learning theory.
\newblock In {\em 33rd Symposium on Theoretical Aspects of Computer Science},
  2016.

\bibitem[LZ92]{lz-tmllc-92}
S.~Lange and T.~Zeugmann.
\newblock Types of monotonic language learning and their characterization.
\newblock In {\em Proc. 5th Annual ACM Workshop on Comput. Learning Theory},
  pages 377--390, New York, NY, 1992. ACM Press.

\bibitem[LZ96]{Lan-Zeu:j:96}
S.~Lange and T.~Zeugmann.
\newblock Incremental learning from positive data.
\newblock {\em Journal of Computer and System Sciences}, 53:88--103, 1996.

\bibitem[LZZ08]{lange2008learning}
S.~Lange, T.~Zeugmann, and S.~Zilles.
\newblock Learning indexed families of recursive languages from positive data:
  A survey.
\newblock {\em Theoretical Computer Science}, 397(1):194--232, 2008.

\bibitem[Odi99]{Odi:b:99}
P.~Odifreddi.
\newblock {\em Classical Recursion Theory}, volume~II.
\newblock Elsivier, Amsterdam, 1999.

\bibitem[OSW82]{Osh-Sto-Wei:j:82:strategies}
D.~Osherson, M.~Stob, and S.~Weinstein.
\newblock Learning strategies.
\newblock {\em Information and Control}, 53:32--51, 1982.

\bibitem[OSW86]{STL1}
D.~Osherson, M.~Stob, and S.~Weinstein.
\newblock {\em Systems that Learn: {A}n Introduction to Learning Theory for
  Cognitive and Computer Scientists}.
\newblock MIT Press, Cambridge, Mass., 1986.

\bibitem[RC94]{Roy-Cas:b:94}
J.~Royer and J.~Case.
\newblock {\em Subrecursive Programming Systems: Complexity and Succinctness}.
\newblock Research monograph in {\em Progress in Theoretical Computer Science}.
  {Birkh\"auser}~Boston, 1994.

\bibitem[Rog67]{Rog:b:67}
H.~Rogers.
\newblock {\em Theory of Recursive Functions and Effective Computability}.
\newblock McGraw Hill, New York, 1967.
\newblock Reprinted, MIT Press, 1987.

\bibitem[Sha15]{shamir2015sample}
O.~Shamir.
\newblock The sample complexity of learning linear predictors with the squared
  loss.
\newblock {\em The Journal of Machine Learning Research}, 16(1):3475--3486,
  2015.

\bibitem[SSBD14]{shalev2014understanding}
S.~Shalev-Shwartz and S.~Ben-David.
\newblock {\em Understanding machine learning: From theory to algorithms}.
\newblock Cambridge university press, 2014.

\bibitem[Wie76]{wiehagen1976limes}
R.~Wiehagen.
\newblock Limes-erkennung rekursiver funktionen durch spezielle strategien.
\newblock {\em J. Inf. Process. Cybern.}, 12 (1-2):93--99, 1976.

\bibitem[Wie91]{Wie:c:91}
R.~Wiehagen.
\newblock A thesis in inductive inference.
\newblock In {\em Nonmonotonic and Inductive Logic, 1st International Workshop,
  Proc.}, pages 184--207, 1991.

\end{thebibliography}

\end{document}